\documentclass[11pt]{article}

\usepackage{arxiv}

\usepackage[utf8]{inputenc}
\usepackage[T1]{fontenc}
\usepackage{hyperref}
\usepackage{url}
\usepackage{microtype}

\usepackage{graphicx}
\usepackage{subcaption}
\usepackage{diagbox}
\usepackage{color}

\usepackage{amsmath,mathrsfs,amsfonts,amssymb,mathtools}
\usepackage{amsthm}
\usepackage{bm}
\usepackage{dsfont}

\newtheorem{assumption}{Assumption}
\newtheorem{proposition}{Proposition}
\newtheorem{lemma}{Lemma}
\newtheorem{remark}{Remark}
\newtheorem{theorem}{Theorem}
\newtheorem{corollary}{Corollary}

\usepackage{algorithm}
\usepackage[noend]{algpseudocode}

\usepackage{paralist}
\usepackage{todonotes}
\usepackage[normalem]{ulem}

\usepackage{natbib}

\newcommand{\argmax}{\operatorname*{arg\,max}}
\newcommand{\argmin}{\operatorname*{arg\,min}}
\DeclareMathOperator{\IF}{IF}

\title{Concave Statistical Utility Maximization Bandits via Influence-Function Gradients}

\author{
	Matías Carrasco \\
	Universidad ORT Uruguay \\
	Montevideo, Uruguay \\
	\texttt{carrasco\_m@ort.edu.uy} \\
	\And
	Alejandro Cholaquidis \\
	Facultad de Ciencias \\
	Universidad de la Rep\'ublica \\
	Montevideo, Uruguay \\
	\texttt{acholaquidis@cmat.edu.uy}
}

\date{}

\begin{document}
	
	\maketitle

\begin{abstract}%
\footnotesize
We study stochastic multi-armed bandits in which the objective is a statistical functional of the long-run reward distribution, rather than expected reward alone. Under mild continuity assumptions, we show that the infinite-horizon problem reduces to optimizing over stationary mixed policies: each weight vector \(w\) on the simplex induces a mixture law \(P^w\), and performance is measured by the concave utility \(U(w)=\mathfrak U(P^w)\).

For differentiable statistical utilities, we use influence-function calculus to derive stochastic gradient estimators from bandit feedback. This leads to an entropic mirror-ascent algorithm on a truncated simplex, implemented through multiplicative-weights updates and plug-in estimates of the influence function. We establish regret bounds that separate the mirror-ascent optimization error from the bias caused by estimating the influence function. The framework is developed for general concave distributional utilities and illustrated through variance and Wasserstein objectives, with numerical experiments comparing exact and plug-in influence-function implementations.
\end{abstract}


\section{Introduction}

Sequential decision-making under uncertainty is a central topic in machine learning. In the classical stochastic multi-armed bandit problem, a learner repeatedly selects one of finitely many arms and seeks to maximize cumulative expected reward while balancing exploration and exploitation; see \cite{BubeckCesaBianchi2012,lattimore_bandit_2020}. In many applications, however, the quality of a sampling strategy depends on the distribution of the collected rewards, through a statistical functional $\mathfrak U$, and not only on their mean. This occurs, for example, when one seeks to promote variability, control sensitivity to tails, or match a prescribed target distribution. Such objectives fall outside the standard mean-reward bandit framework; see, e.g., \cite{cassel2023general}.

In particular, for broad classes of nonstandard objectives, an optimal policy need not concentrate on a single best arm, but may instead involve a nontrivial mixture of arms with specific asymptotic sampling proportions. This observation suggests a natural distributional viewpoint: under a stationary randomized policy with weights $w$, the induced reward law is the mixture of the arm distributions associated with $w$, and the learning problem becomes the optimization of a statistical utility $\mathfrak U(P^w)$ over the simplex. However, for genuinely distribution-dependent utilities, classical mean-based bandit tools no longer apply directly; this has motivated a substantial literature on risk-aware and non-mean bandit criteria
\citep{sani_risk-aversion_2012,galichet2013exploration,kagrecha2019distribution,cassel2023general}.

In this work, we study a general class of bandit problems in which the objective is a statistical utility of the reward distribution induced by the learner's sampling strategy. We focus on the concave setting, which covers a broad family of distribution-sensitive criteria arising in machine learning and statistics, including the two main examples developed in the paper: variance-based objectives and Wasserstein-type objectives that seek to match a target distribution. From this perspective, the goal is to learn the mixture of arms whose induced reward law is most desirable according to the chosen criterion.

This distributional viewpoint is motivated by several types of applications. In educational assessment and clinical adaptive testing, heterogeneity across individuals is central: different response patterns can reveal different abilities, learning needs, or levels of impairment
\cite{schochet2014understanding,gibbons2016computerized}. In medicine, related motivations arise from disease heterogeneity \cite{zanin2018understanding,wallstrom2013biomarker}. Related ideas also arise within machine learning itself: Rezaei et al.\ show that, for diversity-oriented evaluation scores of generative models, the optimal decision may be a mixture of models rather than a single one \cite{rezaei2025be}. Wasserstein-type objectives provide a complementary example, since they naturally encode the goal of matching an induced distribution to a target one. This principle is central to applications of optimal transport in machine learning, see e.g.\ \cite{peyre2019computational}, and appears explicitly in Wasserstein losses, see e.g.\ \cite{frogner2015learning,arjovsky2017wasserstein}.

The main technical difficulty is that the usual bandit estimators exploit the fact that the mean objective decomposes into arm-wise contributions. This decomposition is lost for nonlinear objectives of the form $\mathfrak U(P^w)$: the first-order variation depends on the full mixture law $P^w$ and on how each arm distribution $P^k$ perturbs it. To handle this dependence, we rely on influence-function calculus; see, e.g., \cite{hampel1974,huber1981robust}. Under suitable differentiability assumptions, the local variation of the utility can be represented through an influence function, which in turn leads to a principled stochastic mirror-ascent procedure on the simplex \cite{cesa2021online}. We then develop a fully implementable plug-in version of this method, in which the unknown oracle score is replaced by an empirical estimate built from the observed rewards. Moreover, the influence-function representation naturally identifies a centered representative of the simplex gradient: subtracting the mixture average of the influence-function score acts as a baseline, analogously to baseline choices in policy-gradient estimators, see e.g.\ \cite{mei2022role}. This connection between influence functions and sequential learning is one of the main conceptual contributions of the paper.

We also provide a regret analysis that separates the optimization mechanism from the statistical error induced by score estimation, and we show empirically, on variance and Wasserstein examples, that the plug-in method remains competitive with its oracle counterpart across a range of synthetic instances.

Related work on bandits with nonstandard objectives includes risk-sensitive and beyond-cumulative formulations, where the performance criterion depends on more than the mean reward \cite{cassel2023general}. More recently, it has been shown that for broad families of risk- or preference-based objectives, optimal policies may require nontrivial mixtures of arms rather than repeated play of a single arm \cite{tatli2025risk,tatli2025preference}. Our contribution is complementary: rather than focusing on a specific objective, we develop a general first-order optimization framework for differentiable concave statistical utilities of the induced reward law. From the algorithmic viewpoint, our method is also related in spirit to mirror-descent and exponentiated-gradient methods on the simplex \cite{JuditskyNazinTsybakovVayatis2008ifac,ShalevShwartz2012,BeckTeboulle2003,cesa2021online}, but the gradient objects, the statistical estimators, and the analysis required here are specific to nonlinear utilities of mixture distributions.

The remainder of the paper is organized as follows. Section~\ref{sec:stat_utility_bandits} introduces the statistical-utility bandit model, proves the reduction to stationary mixed policies, and formulates the induced optimization problem over the simplex. Section~\ref{sec:simplex_gradients} sets up the differential calculus used to represent gradients on the simplex and relates these gradients to the softmax parametrization. Section~\ref{sec:if_gradients} derives the influence-function representation of the simplex gradient and introduces the corresponding unbiased and plug-in estimators under bandit feedback. Section~\ref{sec:algorithm} presents the resulting entropic mirror-ascent algorithm on the truncated simplex. Section~\ref{sec:regret_analysis} establishes regret bounds and separates the optimization error from the bias induced by estimating the influence function. Section~\ref{sec:applications} verifies the assumptions for two representative utilities, namely the variance utility in Section~\ref{sec:variance_utility} and the Wasserstein utility in Section~\ref{sec:wasserstein_utility}. Section~\ref{sec:experiments} reports numerical experiments comparing exact and plug-in influence-function implementations, and Section~\ref{sec:conclusion} concludes.

\subsection{Notation}

We use the following notation throughout the paper. Deterministic functions are denoted in standard math font (e.g.\ $U$, $s$, $J$), and deterministic statistical functionals in math gothic font (e.g.\ $\mathfrak U$). Deterministic points in Euclidean space are denoted by lower-case letters (e.g.\ $w,h,u$), and deterministic vectors in bold font (e.g.\ $\mathbf v$, $\mathbf g,\mathbf 1$). Random variables are denoted in sans serif font (e.g.\ $\mathsf A_t,\mathsf R_t,\mathsf U_\infty^\pi$), random points in bold sans serif lower-case font (e.g.\ $\mathsf w_t$, $\mathsf h_t$), and random vectors in bold sans serif font (e.g.\ $\boldsymbol{\mathsf G}_t,\boldsymbol{\mathsf H}_t,\boldsymbol{\mathsf B}_t$). 
With a slight abuse of notation, we use the same symbol for a probability measure and for its cumulative distribution function. Thus \(P(B)\) denotes the mass assigned to a Borel set \(B\subset \mathbb{R}\), while \(P(x)\) denotes the value of the corresponding cdf at \(x\in\mathbb{R}\).

\section{Statistical utility bandits}\label{sec:stat_utility_bandits}

In this section we introduce the statistical and bandit theoretic framework underlying our notion of utility maximizing bandits. Our focus is on policies whose goal is not to maximize expected reward, but rather to optimize a statistical functional of the long-run mixture distribution of observed rewards. Classical multi-armed bandit theory is recovered as the special case where the utility function is the mean. Our exposition follows the standard bandit notation (see for instance \cite{lattimore_bandit_2020}) combined with general statistical functionals ideas.

\subsection{Bandit sampling model}

We work in a fixed probability space \((\Omega,\mathcal{F},\mathbb{P})\). Arms are indexed by the finite set \(\mathcal{A}=\{1,\dots,K\}\). In each round \(t\geqslant 1\) an action \(\mathsf{A}_t\in \mathcal{A}\) is selected. Conditioned on \(\mathsf{A}_t=k\) the reward \(\mathsf{R}_t\in\mathbb{R}\) is drawn from a fixed distribution \(P^{k}\) with finite second moment. 
Formally
\(
\mathbb{P}(\mathsf{R}_t\in B \mid \mathcal{F}_t^{\,\text{act}})
=\mathbb{P}(\mathsf{R}_t\in B \mid \mathsf{A}_t)
=P^{\mathsf{A}_t}(B)
\) for every Borel set \(B\subset\mathbb{R}\), where
\[
\mathcal{F}_t^{\,\text{act}}=\sigma(\mathsf{A}_1,\mathsf{R}_1,\dots,\mathsf{A}_t),\qquad \text{ and }\qquad 
\mathcal{F}_t^{\,\text{rew}}=\sigma(\mathsf{A}_1,\mathsf{R}_1,\dots,\mathsf{A}_t,\mathsf{R}_t).
\]
are the history filtrations up to action $\mathsf{A}_t$ and reward $\mathsf{R}_t$ respectively. In particular, conditionally on the action sequence, the rewards are independent and \(\mathsf{R}_t \mid \{\mathsf{A}_t=k\} \sim P^k\).

A policy is a sequence of conditional action distributions
\[
\pi_t (k)=\mathbb{P}\left(\mathsf{A}_t=k\mid \mathcal{F}_{t-1}^{\,\text{rew}}\right),\quad k=1,\dots,K,\quad t\geqslant 1,
\]
where $\mathcal{F}_{0}^{\,\text{rew}}$ is the trivial $\sigma$-algebra. Given a policy \(\pi\), the empirical distribution of rewards up to time $t$ is
\begin{equation*}
	\mathsf{P}_t^\pi(x)=\frac{1}{t}\sum_{s=1}^t \mathds{1}_{\{\mathsf{R}_s\le x\}}.
\end{equation*}
For each individual arm \(k\), write
\(
\mathsf{N}_{t,k}=\sum_{s=1}^t \mathds{1}_{\{\mathsf{A}_s=k\}}
\)
for the number of pulls up to time \(t\). Define the empirical arm-specific distribution (up to time $t$) as follows: if $\mathsf N_{t,k}=0$ we select any probability distribution, otherwise 
\begin{equation}\label{eq:arm_empirical}
	\mathsf{P}_t^{k}(x)=\frac{1}{\mathsf{N}_{t,k}}\sum_{s=1}^t \mathds{1}_{\{\mathsf{A}_s=k\}}\mathds{1}_{\{\mathsf{R}_s\leqslant x\}},  \qquad\mathsf{N}_{t,k}>0.
\end{equation}
Then
\begin{equation}\label{eq:decomp_empirical}
	\mathsf{P}_t^\pi=\sum_{k=1}^K \frac{\mathsf{N}_{t,k}}{t}\mathsf{P}_t^{k}.
\end{equation}

\subsection{Distances and mixture distributions}

In what follows \(\mathcal{P}_2(\mathbb R)\) denotes the set of all Borel probability distributions on the real line with finite second moment. We endow \(\mathcal{P}_2(\mathbb R)\) with a distance \(d\) that satisfies the following two continuity conditions.
\begin{enumerate}
	\item Empirical distributions converge: If $\mathsf{P}_n$ is the empirical distribution of $n$ iid random variables with distribution $P$, then
	\begin{equation}\label{eq:dist_continuity_1}
		d(\mathsf{P}_n,P)\to 0 \quad\text{a.s.}
	\end{equation}
	when $n\to\infty$.
	\item The mixtures are continuous in weights and components: If $|\alpha_n - \alpha_n'| \to 0$ in $[0,1]$, $d(P_n, P_n') \to 0$, $d(Q_n, Q_n') \to 0$, and \(
	\sup_n d(P_n',Q_n')<\infty
	\), then
	\begin{equation}\label{eq:dist_continuity_2}
		d\Big(\alpha_n P_n + (1-\alpha_n)Q_n, \alpha_n'P_n'+(1-\alpha_n')Q_n'\Big)\to 0.
	\end{equation}
\end{enumerate}

Condition \eqref{eq:dist_continuity_2} is an abstract continuity requirement on mixtures. The next lemma gives a convenient sufficient criterion in terms of joint convexity of \(d^p\) and continuity with respect to the mixing coefficient. These conditions are straightforward to check in important examples, such as Wasserstein distances; see Remark~\ref{rem:wasserstein_mixtures}.

\begin{proposition}\label{prop:mixture_continuity}
	Assume that $d(\cdot, \cdot)^p$ is jointly convex for some $p\in [1,\infty)$, and that there exists a modulus of continuity
	\(
	\eta(x,\delta)
	\)
	such that
	\begin{enumerate}
		\item 
		\(
		d\left(\alpha P+(1-\alpha)Q,\alpha'P+(1-\alpha')Q\right)
		\leqslant
		\eta\left(d(P,Q),|\alpha-\alpha'|\right)
		\);
		\item for every $M<\infty$:
		\(
		\sup_{0\le x\le M}\eta(x,\delta)\to 0
		\)
		as $\delta\to 0$.
	\end{enumerate}
	Then $d$ satisfies condition \eqref{eq:dist_continuity_2}.
\end{proposition}

\begin{proof}
	Let $|\alpha_n - \alpha_n'| \to 0$ in $[0,1]$, $d(P_n, P_n') \to 0$, $d(Q_n, Q_n') \to 0$, and 
	\(
	M=\sup_n d(P_n',Q_n')<\infty
	\).
	We have
	\begin{align*}
		d\left(\alpha_n P_n+(1-\alpha_n)Q_n, \alpha_n'P_n'+(1-\alpha_n')Q_n'\right)
		&\leqslant
		d\left(\alpha_n P_n+(1-\alpha_n)Q_n,\alpha_n P_n'+(1-\alpha_n)Q_n'\right)\\
		&+
		d\left(\alpha_n P_n'+(1-\alpha_n)Q_n',\alpha_n'P_n'+(1-\alpha_n')Q_n'\right).
	\end{align*}
	We treat the two terms of the RHS separately. For the first term
	\begin{align*}
		d\left(\alpha_n P_n+(1-\alpha_n)Q_n, \alpha_n P_n'+(1-\alpha_n)Q_n'\right)^p
		&\leqslant
		\alpha_n d(P_n,P_n')^p+(1-\alpha_n)d(Q_n,Q_n')^p,
	\end{align*}
	and therefore tends to $0$. For the second term
	\[
	d\left(\alpha_n P_n'+(1-\alpha_n)Q_n', \alpha_n'P_n'+(1-\alpha_n')Q_n'\right)
	\leqslant
	\eta\left(d(P_n',Q_n'),|\alpha_n-\alpha_n'|\right).
	\]
	Since $M<\infty$ and $|\alpha_n-\alpha_n'|\to0$, this term tends to $0$ as well.
\end{proof}

\begin{remark}\label{rem:wasserstein_mixtures}
	For $d=W_p$ the Wasserstein $p$-distance on $\mathcal P_p(\mathbb R)$ (the set of probability distributions with finite $p$-moment), joint convexity of $d^p$ holds, and the continuity in the mixing coefficient is controlled by
	\[
	W_p\left(\alpha P+(1-\alpha)Q,\alpha'P+(1-\alpha')Q\right)
	\leqslant
	|\alpha-\alpha'|^{1/p}W_p(P,Q),
	\]
	see for instance \cite{Villani2009}.
\end{remark}

\subsection{Statistical utilities and stationary polices}

A statistical utility is a functional
\(
\mathfrak U:\mathcal P\to\mathbb R
\)
that is continuous with respect to \(d\).
Given a policy \(\pi\) define its asymptotic utility (or infinite horizon utility)
\[
\mathsf{U}^{\pi}_{\infty}=
\liminf_{t\to\infty} \mathfrak{U}\left(\mathsf{P}_t^\pi\right).
\]
We are interested in policies  $\pi^\star$ that maximize the expected asymptotic utility, that is,
\[
\pi^\star\in\argmax_{\pi} \mathbb E[\mathsf{U}^{\pi}_{\infty} ].
\]

For any $w$ in the probability simplex
\(
\Delta=\left\{w\in\mathbb{R}^K: w_k\geqslant 0, \sum_k w_k=1\right\}
\)
define the mixture reward distribution
\begin{equation}\label{eq:mixture}
	P^w=\sum_{k=1}^K w_k P^k.
\end{equation}

The following structural result shows that we can restrict ourselves to stationary policies without loss of optimality for the infinite-horizon utility maximization problem. The proof follows the same ideas used in \cite{cassel2023general}.

\begin{proposition} 
	\label{prop:optimal_stationary}
	There exists an optimal stationary policy for the infinite-horizon problem.
	More precisely
	\[
	\sup_{\pi}\mathbb E[\mathsf U^{\pi}_{\infty}]
	=
	\sup_{w\in\Delta}\mathfrak U(P^w).
	\]
	Moreover, for any maximizer $w^\star\in\argmax_{w\in\Delta}\mathfrak U(P^w)$ the stationary policy $\pi_t\equiv w^\star$ is optimal.
\end{proposition}

\begin{proof}
	Fix an arbitrary policy $\pi$.
	For each arm $k\in\{1,\dots,K\}$ define the total number of pulls
	\[
	\mathsf N_{\infty,k}=\sup_{t\geqslant 1}\mathsf N_{t,k}\in\mathbb N\cup\{\infty\}.
	\]
	For $m\geqslant 1$, let $\mathsf T_{k,m}$ be the (possibly infinite) random time of the $m$-th pull of arm $k$
	\[
	\mathsf T_{k,m}=\inf\{t\geqslant 1: \mathsf N_{t,k}=m\}\in\mathbb N\cup\{\infty\}.
	\]
	On the event $\{\mathsf N_{\infty,k}=\infty\}$ we have $\mathsf T_{k,m}<\infty$ for all $m\geqslant 1$.
	On this event define the pull-indexed rewards
	$\mathsf X_{k,m}=\mathsf R_{\mathsf T_{k,m}}$, $m\geqslant 1.$
	Then, on $\{\mathsf N_{\infty,k}=\infty\}$ the sequence $(\mathsf X_{k,m})_{m\geqslant 1}$ is iid with law $P^k$.
	Let $\mathsf Q_{k,m}$ denote the empirical distribution of $\mathsf X_{k,1},\dots,\mathsf X_{k,m}$.
	Under \eqref{eq:dist_continuity_1} we have
	\(
	d\left(\mathsf Q_{k,m},P^k\right)\mathds{1}_{\{\mathsf{N}_{\infty,k}=\infty\}}\to 0 
	\)
	a.s..
	Since $\mathsf P_t^k=\mathsf Q_{k,\mathsf N_{t,k}}$ for every $t$ with $\mathsf N_{t,k}>0$, then
	\begin{equation*}\label{eq:empirical_converges}
		d\left(\mathsf P_{t}^k,P^k\right)\mathds{1}_{\{\mathsf{N}_{\infty,k}=\infty\}}\to 0
		\quad \text{a.s.}
	\end{equation*}
	For each $t\geqslant 1$ define the (random) sampling proportions
	\[
	\mathsf w_t
	=
	\left(\frac{\mathsf N_{t,1}}{t},\dots,\frac{\mathsf N_{t,K}}{t}\right)\in\Delta.
	\]
	Since $\Delta$ is compact, there exists a subsequence $\mathsf t_n\uparrow\infty$
	and random weights $\mathsf{w}\in\Delta$ such that
	$\mathsf w_{\mathsf t_n}\to \mathsf{w}\in \Delta$ a.s.. 
	On the event $\left\{\mathsf N_{\infty,k}<\infty\right\}$ we have $\mathsf N_{\mathsf t_n,k}/\mathsf t_n\to0$ and $\mathsf P_{\mathsf t_n}^k = \mathsf P_{\mathsf N_{\infty,k},k}$ for all $n$ such that $\mathsf t_n\geqslant \mathsf N_{\infty,k}$, that is it becomes constant after the last pull.
	
	Putting it all together we have
	\begin{equation}\label{eq:empirical_both_cases}
		\begin{aligned}
			d\left(\mathsf P_{\mathsf t_n}^k,P^k\right)\mathds{1}_{\{\mathsf{N}_{\infty,k}=\infty\}}+d\left(\mathsf P_{\mathsf t_n}^k,\mathsf P_{\mathsf N_{\infty,k},k}\right)\mathds{1}_{\{\mathsf{N}_{\infty,k}<\infty\}}\to 0\text{ a.s.}\\
			\text{and}\\
			\frac{\mathsf N_{\mathsf t_n,k}}{\mathsf t_n}\to
			\begin{cases}
				\mathsf w_k\geqslant 0 & \text{on }\left\{\mathsf N_{\infty,k}=\infty\right\}\\
				\mathsf w_k=0 & \text{on }\left\{\mathsf N_{\infty,k}<\infty\right\}
			\end{cases}
		\end{aligned}
	\end{equation}
	From decomposition \eqref{eq:decomp_empirical},
	\eqref{eq:empirical_both_cases} and the mixture continuity assumption \eqref{eq:dist_continuity_2}, we get
	\(
	d\left(\mathsf P_{\mathsf t_n}^\pi,P^{\mathsf w}\right)\to 0
	\) a.s..
	Since $\mathfrak U$ is continuous w.r.t.\ $d$, this implies that
	$\mathfrak U\left(\mathsf P_{\mathsf t_n}^\pi\right)\to \mathfrak U\left(P^{\mathsf w}\right)$ a.s.. By the definition  of $\liminf$
	\[
	\mathsf U^{\pi}_{\infty}
	=\liminf_{t\to\infty}\mathfrak U\left(\mathsf P_t^\pi\right)
	\leqslant \lim_{n\to\infty}\mathfrak U\left(\mathsf P_{\mathsf t_n}^\pi\right)
	=\mathfrak U\left(P^{\mathsf w}\right)
	\leqslant \sup_{u\in\Delta}\mathfrak U\left(P^u\right) \quad\text{a.s.}
	\]
	Taking expectations yields $\mathbb E\left[\mathsf U^{\pi}_{\infty}\right]
	\leqslant \sup_{u\in\Delta}\mathfrak U\left(P^u\right).$ 
	Since $\pi$ was arbitrary $\sup_{\pi}\mathbb E\left[\mathsf U^{\pi}_{\infty}\right]\leqslant \sup_{u\in\Delta}\mathfrak U\left(P^u\right)$. Notice that since $\Delta$ is compact and $w\mapsto \mathfrak U(P^w)$ is continuous there is a maximizer $w^\star\in\Delta$.
	
	Finally, fix $w\in\Delta$ and consider the stationary policy $\pi_t\equiv w$.
	Under this policy rewards $(\mathsf R_t)_{t\geqslant 1}$ are iid with common
	distribution $P^w$. By the empirical convergence assumption
	\eqref{eq:dist_continuity_1}, we have
	$d\left(\mathsf P_t^\pi,P^w\right)\to 0$ a.s.\ as $t\to\infty.$ 
	Since $\mathfrak U$ is continuous with respect to $d$ it follows that $\mathfrak U\left(\mathsf P_t^\pi\right)\to \mathfrak U(P^w)$ a.s..
	Consequently $\mathsf U_\infty^\pi
	=\liminf_{t\to\infty}\mathfrak U\left(\mathsf P_t^\pi\right)
	=\mathfrak U(P^w)$ a.s. Therefore any maximizer
	$w^\star\in\argmax_{w\in\Delta}\mathfrak U(P^w)$ induces an optimal stationary policy.
\end{proof}

Proposition~\ref{prop:optimal_stationary} shows that it is sufficient to design algorithms that converge to an optimal sampling distribution,
\[
w^\star\in\argmax_{w\in\Delta}\mathfrak U(P^w).
\]
Thus, the learning problem can be viewed as the adaptive optimization of the induced utility \(U(w)=\mathfrak U(P^w)\) over the simplex. Before turning to the algorithmic analysis, we introduce the two main classes of statistical utilities considered in this paper.

\subsection{Concave statistical utilities over mixtures}

A stationary policy that selects arm $k$ with probability $w_k$ induces the reward mixture $P^w$. Such policies are evaluated through the induced utility
\begin{equation}\label{eq:utility_simplex}
	U(w) = \mathfrak U(P^w), \qquad w\in\Delta .
\end{equation}
We focus on the case where the induced utility $U:\Delta\to\mathbb R$ is concave on the simplex.
Concavity is a standard requirement in online first-order optimization. Similar hypotheses appear throughout the mirror-descent and online convex optimization literature; see, for instance, \cite{ShalevShwartz2012,BeckTeboulle2003,lattimore_bandit_2020}.
Many statistical functionals that quantify dispersion or distributional proximity lead to concave utilities over mixtures of arms. We now present two examples and specify appropriate choices of the metric $d$.

\subsubsection{Example 1: Variance utility}\label{sec:ex_variance_utility}

Let $\mathfrak U(P)=\mathbb{V}(\mathsf R)$ where $\mathsf R\sim P$. If $P^k$ has mean $\mu_k$ and variance $\sigma_k^2$ the variance of the mixture is given by
\[
U(w)
=\mathbb{V}(P^w)
=\sum_{k=1}^K w_k \sigma_k^2
+\sum_{k=1}^K w_k(\mu_k-\mu_w)^2,
\qquad
\mu_w=\sum_{k=1}^K w_k\mu_k .
\]
The first term is linear in $w$, while the second is the negative of a convex quadratic. Hence $U$ is concave on $\Delta$. Under bounded rewards the Wasserstein-1 distance is an adequate choice and continuity of the variance functional follows immediately. See Section \ref{sec:variance_utility} for more details.

\subsubsection{Example 2: Wasserstein proximity to a reference distribution}\label{sec:ex_wasserstein_utility}

Fix a reference distribution $Q$ and define
\(
\mathfrak U(P)=-W_2^2(P,Q)
\)
where $W_2$ denotes the Wasserstein-$2$ distance. In this setting we take $d=W_2$. The mapping $P\mapsto W_2^2(P,Q)$ is convex and continuous with respect to $d$. Therefore, the induced utility $U(w)=-W_2^2(P^w,Q)$ is concave on the simplex. See Section \ref{sec:wasserstein_utility} for more details.

In the remainder of the paper, we construct adaptive procedures for this type of objectives using gradient-based methods, influence-function estimators, and mirror-ascent theory; see \cite{BeckTeboulle2003,NemirovskiJuditskyLanShapiro2009,cesa2021online}.

\section{Gradients on the simplex}\label{sec:simplex_gradients}

We write $\mathring{\Delta}$ for the interior of $\Delta$. For an interior point $w\in\mathring{\Delta}$ the tangent space of $\Delta$ at $w$ is
\[
T_w\Delta = \left\{\mathbf v\in\mathbb{R}^K:\mathbf{1}^\top \mathbf v = 0\right\},
\]
where $\mathbf{1} = (1,\ldots,1)^\top$. Let $U:\Delta\to\mathbb{R}$ be a $C^1$ function and $w\in\mathring{\Delta}$. For any tangent direction $\mathbf v\in T_w\Delta$ the directional derivative of $U$ at $w$ in direction $\mathbf v$ is
\[
d_w U[\mathbf v] = \left.\frac{d}{d\epsilon}\,U(w+\epsilon \mathbf v)\right|_{\epsilon=0}.
\]
We call a vector $\mathbf g(w)\in\mathbb{R}^K$ a gradient on the simplex of $U$ at $w$ if
\begin{equation}\label{eq:gradient-simplex-def}
	d_wU[\mathbf v] = \langle \mathbf g(w), \mathbf v\rangle
	\quad\text{for every } \mathbf v\in T_w\Delta,
\end{equation}
where $\langle \cdot,\cdot\rangle$ is the standard Euclidean inner product on $\mathbb{R}^K$. The gradient in $\Delta$ according to this definition is only defined up to the addition of multiples of $\mathbf{1}$, as made more precise in the following.

\begin{lemma}
	Let $U:\mathring{\Delta}\to\mathbb{R}$ be a $C^1$ function and $w\in\mathring{\Delta}$. Then $\mathbf g(w)$ and $\mathbf g'(w)$ satisfy both \eqref{eq:gradient-simplex-def} if and only if
	$\mathbf g(w) - \mathbf g'(w) = \lambda(w)\mathbf{1}$ for some $\lambda:\mathring{\Delta}\to \mathbb{R}$.
\end{lemma}

\begin{proof}
	If $\mathbf v\in T_w\Delta$ then $\langle \mathbf{1}, \mathbf v\rangle = 0$, so
	$\langle \mathbf g(w)+\lambda(w)\mathbf{1}, \mathbf v\rangle = \langle \mathbf g(w), \mathbf v\rangle
	= d_wU[\mathbf v].$ 
	Conversely, suppose that both \(\mathbf g(w)\) and \(\mathbf g'(w)\) satisfy \eqref{eq:gradient-simplex-def}.
	Then for all $\mathbf v\in T_w\Delta$
	$0 = d_w U[\mathbf v] - d_w U[\mathbf v] = \langle \mathbf{g}(w)-\mathbf{g}'(w), \mathbf v\rangle.$
	Taking $\mathbf v=\mathbf e_i- \mathbf e_j\in T_w\Delta$ for any $i,j$ gives
	\(
	(\mathbf g(w)- \mathbf g'(w))_i = (\mathbf g(w)- \mathbf g'(w))_j
	\),
	so all coordinates of $\mathbf g(w)-\mathbf g'(w)$ are equal and thus
	$\mathbf g(w)-\mathbf g'(w) = \lambda(w)\mathbf{1}$.
\end{proof}

A natural representative of the gradient on the simplex is the intrinsic Riemannian gradient of $U$ thought as a $C^1$ function on the manifold $\mathring{\Delta}$. More precisely, the Riemannian gradient $\mathbf g_T(w)$ at $w\in\mathring{\Delta}$ is the unique vector in the tangent space $T_w\Delta$ such that
\(
d_wU[\mathbf v] = \langle \mathbf g_T(w),\mathbf v\rangle 
\)
for all $\mathbf v\in T_w\Delta$. It always exists for a $C^1$ function on the simplex. For background on Riemannian gradients see for instance \cite[Section 2.1]{petersen2006riemannian}.

We will instead work with the centered representative of the gradient on the simplex. Among all gradients on the simplex at $w$, it is the unique that satisfies the centering condition
\begin{equation}\label{eq:centered}
	\langle w, \mathbf{g}_C(w)\rangle = 0.
\end{equation}
From Lemma~\ref{lem:centered-gradient} we know that any gradient on the simplex can be written as $\mathbf{g}(w) = \mathbf g_T(w) + \lambda(w)\mathbf{1},$ for some $\lambda:\mathring{\Delta}\to\mathbb{R}$. Imposing the centering condition $\langle w,\mathbf{g}_C(w)\rangle=0$ gives $0 = \langle w,\mathbf g_T(w) + \lambda(w)\mathbf{1}\rangle
= \langle w,\mathbf g_T(w)\rangle + \lambda(w)\langle w,\mathbf{1}\rangle.$
Because $\langle w,\mathbf{1}\rangle=\sum_k w_k = 1$ we obtain
\(
\lambda(w) = -\langle w,\mathbf g_T (w)\rangle
\),
and hence $\mathbf{g}_C(w)=\mathbf g_T(w) - \langle w,\mathbf g_T(w)\rangle\mathbf{1}.$

\subsection{Chain rule for the softmax parameterization}

Let \(U:\Delta\to\mathbb R\) be a function of class $C^1$ on $\mathring{\Delta}$ and let \(w=s(h)\) with $h\in\mathbb R^K$, where we denote $s:\mathbb R^K\to \mathbb R^K$ the softmax function. Let \(\mathbf g(w)\in\mathbb R^K\) be any representative of the gradient of \(U\) on the simplex and \(J(h) = \mathrm{diag}(w)-w w^\top\) the Jacobian of the softmax. Notice that $J(h)$ is symmetric. The vector $J(h) \mathbf g(w)$ is independent of the choice of $\mathbf g(w)$. Indeed, it is easy to see that the constant vector is in the kernel of $J(h)$, i.e.\ \(J(h)\mathbf 1=0\). Therefore, if $\mathbf g'(w)=\mathbf g(w)+\lambda(w)\mathbf 1$ is any other representative of the gradient in the simplex we have $J(h) \mathbf g'(w) = J(h) \mathbf g(w) + \lambda(w) J(h) \mathbf 1= J(h) \mathbf g(w)$.

The following lemma states that $J(h) \mathbf g(w)$ is equal to the standard Euclidean gradient $\nabla(U\circ s)$ of the composition $U\circ s$.

\begin{lemma}\label{lem:chain_rule}
	The gradient of the composition \((U\circ s):\mathbb R^K\to\mathbb R\) in logit coordinates at $h$ is the vector \(\nabla(U\circ s)(h)=J(h) \mathbf g(w)\), \(w=s(h)\), where $\mathbf g(w)$ is any gradient on the simplex of $U$.
\end{lemma}
\begin{proof}
	For any direction \(\mathbf u\in\mathbb R^K\) we have
	\(
	d_h(U\circ s)[\mathbf u] = d_wU(w)[J(h)\mathbf u]
	\)
	with $w=s(h)$. By the definition of a gradient on the simplex \(\mathbf g(w)\) we then have
	\(
	d_h(U\circ s)[\mathbf u]= \langle \mathbf g(w),J(h)\mathbf u\rangle
	= \langle J(h) \mathbf g(w),\mathbf u\rangle.
	\)
	Then $\langle \nabla(U\circ s)(h),\mathbf u\rangle =\langle J(h) \mathbf g(w),\mathbf u\rangle$ for all $\mathbf u \in \mathbb{R}^K$ which concludes the proof.
\end{proof}

\section{Influence functions and gradient estimation}\label{sec:if_gradients}

We now consider functionals of the form $U(w) = \mathfrak{U}(P^w)$ as in \eqref{eq:utility_simplex}, where  $P^w$ is a mixture of $K$ fixed distributions $(P^k)_{k=1}^K$ with weights $w\in\Delta$ as in \eqref{eq:mixture}. In this section we derive stochastic gradient estimators for the utility
\(
U(w)
\)
based on a single bandit draw. The key idea is that when the statistical functional $\mathfrak U$ is differentiable in an appropriate sense, small perturbations of the mixture distribution $P$ admit a first-order expansion with respect to the influence function of $\mathfrak U$. Bandit sampling then provides an estimate of the logit-space gradient $\nabla(U\circ s)(h)$.

\subsection{Differentiability and influence functions}

A perturbation of the weights in a tangent direction corresponds to a signed perturbation of the mixture law. Under the following assumption, the associated influence function gives a centered gradient representative on the simplex.

\begin{assumption}
	\label{ass:A1}
	Fix $w\in\mathring\Delta$ and let $P^w=\sum_{i=1}^K w_i P^i$. There exists a measurable function $\IF\left[P^w\right]$ with $\mathbb E_{P^w}[|\IF[P^w](\mathsf R)|]<\infty$  and $\mathbb E_{P^w}[\IF[P^w](\mathsf R)]=0,$ such that for every $\mathbf v\in T_w\Delta$  the directional derivative exists and satisfies
	\[
	\lim_{\epsilon\to 0}\frac{\mathfrak U\left(P^w+\epsilon H\right)-\mathfrak U\left(P^w\right)}{\epsilon}
	=
	\int_{\mathbb{R}} \IF[P^w](r)dH(r),
	\]
	where $H=\sum_{i=1}^K v_i P^i$
\end{assumption}

\begin{proposition}\label{prop:directional-derivative}
	Under Assumption~\ref{ass:A1}, for any \(w\in\mathring{\Delta}\) and any tangent direction \(\mathbf v\in T_w\Delta\), the directional derivative of \(U(w)\) at \(w\) in the direction \(\mathbf v\) exists and satisfies
	\begin{equation}\label{eq:directional-U}
		\left.\frac{d}{d\epsilon}\,U(w+\epsilon \mathbf v)\right|_{\epsilon=0}
		= \sum_{i=1}^K v_i\,\mathbb{E}_{P^i}[\IF[P^w](\mathsf{R})].
	\end{equation}
\end{proposition}

\begin{proof}
	Fix \(w\in\mathring{\Delta}\) and \(\mathbf v\in T_w\Delta\). Define the signed measure $H \coloneqq \sum_{i=1}^K v_i P^i,$ so that for any \(\epsilon\) small enough such that \(w+\epsilon \mathbf v\in\mathring\Delta\),
	\[
	P^{w+\epsilon \mathbf v}=\sum_{i=1}^K (w_i+\epsilon v_i)P^i
	= P^w+\epsilon \sum_{i=1}^K v_iP^i
	= P^w+\epsilon H.
	\]
	Then
	\[
	\frac{U(w+\epsilon \mathbf v)-U(w)}{\epsilon}
	=
	\frac{\mathfrak U(P^{w+\epsilon \mathbf v})-\mathfrak U(P^w)}{\epsilon}
	=
	\frac{\mathfrak U(P^w+\epsilon H)-\mathfrak U(P^w)}{\epsilon}.
	\]
	By Assumption~\ref{ass:A1}, the right-hand side converges as \(\epsilon\to 0\) to
	\(
	\int \IF[P^w](r)dH(r)
	\). Therefore, 
	\[
	\int \IF[P^w](r)dH(r)
	=
	\sum_{i=1}^K v_i \int \IF[P^w](r)dP^i(r)
	=
	\sum_{i=1}^K v_i\mathbb{E}_{P^i}[\IF[P^w](\mathsf R)],
	\]
	where the expectations are well-defined since \(\mathbb E_{P^w}[|\IF[P^w](\mathsf R)|]<\infty\). This proves \eqref{eq:directional-U}.
\end{proof}

Under Assumption~\ref{ass:A1}, for any $w\in\mathring{\Delta}$ we can define the vector $\mathbf g(w)\in\mathbb{R}^K$ component-wise by
\begin{equation}\label{eq:ic_gradient_def}
	(\mathbf g(w))_k = \mathbb{E}_{P^k}[\IF[P^w](\mathsf{R})], \qquad k=1,\dots,K.
\end{equation}

\begin{lemma}\label{lem:centered-gradient}
	Under Assumption \ref{ass:A1}:
	\begin{enumerate}
		\item $\mathbf g(w)$ is a valid representative of the gradient of $U$ on the simplex, i.e.,  satisfies \eqref{eq:gradient-simplex-def}.
		\item $\mathbf g(w)$ satisfies the centering condition \eqref{eq:centered}, namely $\langle w,\mathbf g(w)\rangle=0$.
	\end{enumerate}
	Therefore $\mathbf{g}(w)=\mathbf{g}_C(w)$.
\end{lemma}

\begin{proof}
	\begin{enumerate}
		\item By Proposition~\ref{prop:directional-derivative} for every $\mathbf v\in T_w\Delta$ 
		\(
		d_w U[\mathbf v]
		=\sum_{k=1}^K v_k\,\mathbb{E}_{P^k}[\IF[P^w](\mathsf{R})]
		=\left\langle \mathbf g(w),\mathbf v\right\rangle
		\).
		\item Using $P^w=\sum_{k=1}^K w_k P^k$ and $\mathbb E_{P^w}[\IF[P^w](\mathsf R)]=0$,
		\[
		\langle w,\mathbf g(w)\rangle
		=\sum_{k=1}^K w_k\mathbb{E}_{P^k}[\IF[P^w](\mathsf{R})]
		=\mathbb{E}_{P^w}[\IF[P^w](\mathsf{R})]
		=0.
		\]
	\end{enumerate}
	Since $\mathbf g_C$ is the unique gradient in the simplex satisfying \eqref{eq:centered} we have $\mathbf{g}(w)=\mathbf{g}_C(w)$.
\end{proof}

\subsection{An unbiased influence-function based estimator of the gradient}

Assume $w\in\mathring{\Delta}$ is fixed. Draw an action $\mathsf{A}\sim w$ and then a reward $\mathsf{R}\mid\{\mathsf{A}=k\}\sim P^{k}$. Define the per-arm estimator
\begin{equation*}\label{eq:unbiassed_estimator}
	\mathsf{H}_k(w)
	=
	\left(\mathds{1}_{\{\mathsf{A}=k\}}-w_k\right)\IF[P^w](\mathsf{R}),
	\qquad k=1,\dots,K,
\end{equation*}
and let $\boldsymbol{\mathsf{H}}(w)=(\mathsf{H}_1(w),\dots,\mathsf{H}_K(w))^\top$. A direct computation shows that this random vector is an unbiased estimator of the logit-space gradient.

\begin{proposition}\label{prop:IF-unbiased}
	If $\mathsf{A}\sim w$ and $\mathsf{R}\mid \mathsf{A}\sim P^{\mathsf{A}}$, then
	\begin{equation}\label{eq:unbias_gradient}
		\mathbb E\left[\boldsymbol{\mathsf{H}}(w)\right]
		=
		\mathrm{diag}(w)\mathbf g_C(w)
		=
		\nabla (U\circ s)(h),
		\qquad w=s(h),
	\end{equation}
	where $\mathbf g_C(w)$ is defined in \eqref{eq:ic_gradient_def} and $s$ is the softmax map.
\end{proposition}

\begin{proof}
	Fix $k\in\{1,\dots,K\}$. \(\mathbb E\left[\mathsf{H}_k(w)\right]=\mathbb E\left[\mathds{1}_{\{\mathsf{A}=k\}}\IF[P^w](\mathsf{R})\right]-w_k\mathbb E\left[\IF[P^w](\mathsf{R})\right].\) For the first term, condition on $\mathsf{A}$:
	\begin{align*}
		\mathbb E\left[\mathds{1}_{\{\mathsf{A}=k\}}\IF[P^w](\mathsf{R})\right]
		& =
		\mathbb E\left[\mathds{1}_{\{\mathsf{A}=k\}}
		\mathbb E\left[\IF[P^w](\mathsf{R})\mid \mathsf{A}\right]\right]\\
		& =
		\mathbb E\left[\mathds{1}_{\{\mathsf{A}=k\}}
		\sum_{j=1}^K\mathbb E\left[\IF[P^w](\mathsf{R})\mid \mathsf{A}=j\right]\mathds{1}_{\{\mathsf A = j\}}\right]\\
		& =
		\mathbb E\left[\mathds{1}_{\{\mathsf{A}=k\}}
		\mathbb E\left[\IF[P^w](\mathsf{R})\mid \mathsf{A}=k\right]\right]\\
		& =
		w_k\mathbb E_{P^{k}}\left[\IF[P^w](\mathsf{R})\right]=w_k\left(\mathbf g_C(w)\right)_k.
	\end{align*}
	The second term is exactly $\mathbb E_{P^w}\left[\IF[P^w](\mathsf{R})\right]=0$. Therefore $\mathbb E[\mathsf H_k(w)]=w_k(\mathbf g_C(w))_k$ for every $k$, which yields $\mathbb E\left[\boldsymbol{\mathsf H}(w)\right]=\mathrm{diag}(w)\mathbf g_C(w).$ By the centering condition $\mathrm{diag}(w)\mathbf g_C(w)=J(h) \mathbf g_C(w)$ with $w=s(h)$. Then, by Lemma~\ref{lem:chain_rule} we obtain \eqref{eq:unbias_gradient}.
\end{proof}

\subsection{Influence function estimators}\label{sec:if-estimators}

In practice the true influence function $\IF[P^w]$ is unknown and must be replaced by a data-dependent estimator $\widehat{\mathsf{IF}}_t$. Recall that $\mathcal F^{\,\mathrm{rew}}_{t-1}$ denotes the $\sigma$-algebra generated by all past actions and rewards up to time $t-1$. Let $\mathsf w_t$ be a $\mathcal F^{\,\mathrm{rew}}_{t-1}$-measurable sampling distribution at round $t$. Define the plug-in logit estimator
\begin{equation}\label{eq:def_H}
	\widehat{\mathsf{H}}_{t,k}
	=
	\left(\mathds{1}_{\{\mathsf{A}_t=k\}}-\mathsf{w}_{t,k}\right)\widehat{\mathsf{IF}}_{t-1}(\mathsf{R}_t),
	\quad k=1,\dots,K,
\end{equation}
where $\mathsf A_t\sim \mathsf w_t$ and $\mathsf R_t\mid\{\mathsf A_t=k\}\sim P^k$. We assume that the map
\(
(\omega,r)\mapsto \widehat{\mathsf{IF}}_{t-1}(\omega,r)
\)
is $\mathcal F^{\,\mathrm{rew}}_{t-1}\otimes \mathcal B(\mathbb R)$-measurable  and integrable under each $P^k$.

Define the bias vector
\begin{equation}\label{eq:bias}
	\mathsf B_{t,k}
	=
	\Big(\mathbb E_{P^{k}}\left[\widehat{\mathsf{IF}}_{t-1}(\mathsf R)\right]-\mathbb E_{P^{\mathsf{w}_t}}\left[\widehat{\mathsf{IF}}_{t-1}(\mathsf R)\right]\Big)-(\mathbf g_C(\mathsf w_t))_k,
	\qquad
	k=1,\ldots,K,
	\quad
	t\geqslant 1.
\end{equation}
Both $\mathbb E_{P^{k}}\left[\widehat{\mathsf{IF}}_{t-1}(\mathsf R)\right]$ and $\mathbb{E}_{P^{\mathsf{w}_t}}\left[\widehat{\mathsf{IF}}_{t-1}(\mathsf R)\right]$ are random variables. More precisely
\[
\mathbb{E}_{P^k}\left[\widehat{\mathsf{IF}}_{t-1}(\mathsf R)\right]=\mathbb{E}\left[\widehat{\mathsf{IF}}_{t-1}(\mathsf R_t)\mid \mathcal F_{t-1}^{\,\text{rew}}, \mathsf A_t=k\right]:\omega \mapsto \int \widehat{\mathsf{IF}}_{t-1}(\omega,r)\,P^k(dr),
\]
and the same holds for  $\mathbb{E}_{P^{\mathsf{w}_t}}\left[\widehat{\mathsf{IF}}_{t-1}(\mathsf R)\right]$ . Note that $\boldsymbol{\mathsf B}_t$ is $\mathcal F^{\,\mathrm{rew}}_{t-1}$-measurable.

Since $\mathcal F^{\,\mathrm{act}}_{t}=\sigma(\mathcal F^{\,\mathrm{rew}}_{t-1},\mathsf A_t)$, using the tower property we obtain the conditional expectation of $\mathds{1}_{\{\mathsf A_t=k\}}\widehat{\mathsf{IF}}_{t-1}(\mathsf{R}_t)$ for each $k$:
\[
\begin{aligned}
	\mathbb E\left[\mathds{1}_{\{\mathsf A_t=k\}}\widehat{\mathsf{IF}}_{t-1}(\mathsf{R}_t)\mid\mathcal F^{\,\mathrm{rew}}_{t-1}\right]
	&=
	\mathbb E\left[\mathbb E\left[\mathds{1}_{\{\mathsf A_t=k\}}\widehat{\mathsf{IF}}_{t-1}(\mathsf{R}_t)\mid \mathcal F^{\,\mathrm{act}}_{t}\right]\mid\mathcal F^{\,\mathrm{rew}}_{t-1}\right]\\
	&=\mathbb E\left[\sum_{j=1}^K\mathds{1}_{\{\mathsf A_t=j\}}\mathbb E\left[\mathds{1}_{\{\mathsf A_t=k\}}\widehat{\mathsf{IF}}_{t-1}(\mathsf{R}_t)\mid \mathcal F^{\,\mathrm{rew}}_{t-1}, \mathsf A_t=j\right]\mid\mathcal F^{\,\mathrm{rew}}_{t-1}\right]\\
	&=\mathbb E\left[\mathds{1}_{\{\mathsf A_t=k\}}\mathbb E\left[\widehat{\mathsf{IF}}_{t-1}(\mathsf{R}_t)\mid \mathcal F^{\,\mathrm{rew}}_{t-1}, \mathsf A_t=k\right]\mid\mathcal F^{\,\mathrm{rew}}_{t-1}\right]\\
	&=\mathbb E\left[\mathds{1}_{\{\mathsf A_t=k\}}\mathbb E_{P^{k}}\left[\widehat{\mathsf{IF}}_{t-1}(\mathsf R)\right]\mid\mathcal F^{\,\mathrm{rew}}_{t-1}\right]\\
	&=\mathbb E\left[\mathds{1}_{\{\mathsf A_t=k\}}\mid\mathcal F^{\,\mathrm{rew}}_{t-1}\right]\mathbb E_{P^{k}}\left[\widehat{\mathsf{IF}}_{t-1}(\mathsf R)\right]\\
	&=\mathsf w_{t,k}\mathbb E_{P^{k}}\left[\widehat{\mathsf{IF}}_{t-1}(\mathsf R)\right].
\end{aligned}
\]
Therefore
\[
\mathbb E\left[\widehat{\mathsf{H}}_{t,k}\mid\mathcal F^{\,\mathrm{rew}}_{t-1}\right]=\mathsf w_{t,k}\left(\mathbb E_{P^{k}}\left[\widehat{\mathsf{IF}}_{t-1}(\mathsf R)\right]-\mathbb E_{P^{\mathsf{w}_t}}\left[\widehat{\mathsf{IF}}_{t-1}(\mathsf R)\right]\right),
\]
and hence in vector form
\begin{equation}\label{eq:bias_decomp}
	\mathbb E\left[\widehat{\boldsymbol{\mathsf{H}}}_{t}\mid\mathcal F^{\,\mathrm{rew}}_{t-1}\right]
	=
	\text{\rm diag}(\mathsf w_t)\left(\mathbf g_C(\mathsf w_t)+\boldsymbol{\mathsf{B}}_t\right)
	=
	\nabla(U\circ s)(\mathsf h_t)+\text{\rm diag}(\mathsf w_t)\boldsymbol{\mathsf{B}}_t,
\end{equation}
with $\mathsf{w}_t=s(\mathsf h_t)$, where the last equality holds by \eqref{eq:unbias_gradient}.

Define the martingale-difference noise
\[
\boldsymbol{\mathsf{X}}_t
=
\widehat{\boldsymbol{\mathsf{H}}}_{t}
-
\mathbb E\left[\widehat{\boldsymbol{\mathsf{H}}}_{t}\mid\mathcal F^{\,\mathrm{rew}}_{t-1}\right],
\qquad
\mathbb E\left[\boldsymbol{\mathsf{X}}_t\mid\mathcal F^{\,\mathrm{rew}}_{t-1}\right]=0.
\]
Observe that $\widehat{\boldsymbol{\mathsf H}}_{t}$ and $\boldsymbol{\mathsf X}_t$ encode information in logit-space coordinates. We transform to simplex coordinates via
$\widehat{\boldsymbol{\mathsf{G}}}_{t}
=\text{\rm diag}(\mathsf w_t)^{-1}\widehat{\boldsymbol{\mathsf{H}}}_{t},$ $\boldsymbol{\mathsf{Z}}_t
=\text{\rm diag}(\mathsf w_t)^{-1}\boldsymbol{\mathsf{X}}_t.$ 

The following lemma is a standard bias-noise decomposition in stochastic approximation, where the update is expressed as a mean term, a martingale-difference fluctuation, and a (usually small) residual bias; see, e.g., \cite{BenvenisteMetivierPriouret1990,KushnerYin2003,Borkar2008}.

\begin{lemma}\label{lem:IF-decomposition}
	For every $t\geqslant 1$ we have the following noise-bias decomposition
	\[
	\widehat{\boldsymbol{\mathsf{G}}}_{t}
	=
	\mathbf g_C(\mathsf w_t)
	+ \boldsymbol{\mathsf{Z}}_t
	+ \boldsymbol{\mathsf{B}}_t,
	\qquad \text{ and } \qquad 
	\mathbb E\left[\boldsymbol{\mathsf{Z}}_t\mid\mathcal F^{\,\mathrm{rew}}_{t-1}\right]=0.
	\]
\end{lemma}

\begin{proof}
	From \eqref{eq:bias_decomp} 
	\begin{equation}\label{eq:conditional_identity_1}
		\mathrm{diag}(\mathsf w_t)^{-1}
		\mathbb E\left[\widehat{\boldsymbol{\mathsf{H}}}_{t}\mid\mathcal F^{\,\mathrm{rew}}_{t-1}\right]
		=
		\mathbf g_C(\mathsf w_t)+\boldsymbol{\mathsf B}_t.
	\end{equation}
	Since $\boldsymbol{\mathsf Z}_t=\mathrm{diag}(\mathsf w_t)^{-1}\boldsymbol{\mathsf X}_t$, using \eqref{eq:conditional_identity_1} we obtain
	\begin{align*}
		\widehat{\boldsymbol{\mathsf{G}}}_{t}
		&= \mathrm{diag}(\mathsf w_t)^{-1}\widehat{\boldsymbol{\mathsf{H}}}_{t}= \mathrm{diag}(\mathsf w_t)^{-1}\left(\boldsymbol{\mathsf X}_t + \mathbb E\left[\widehat{\boldsymbol{\mathsf{H}}}_{t}\mid\mathcal F^{\,\mathrm{rew}}_{t-1}\right]\right)\\
		&=
		\mathrm{diag}(\mathsf w_t)^{-1}
		\mathbb E\left[\widehat{\boldsymbol{\mathsf{H}}}_{t}\mid\mathcal F^{\,\mathrm{rew}}_{t-1}\right] + \boldsymbol{\mathsf Z}_t=\mathbf g_C(\mathsf w_t)+\boldsymbol{\mathsf Z}_t+\boldsymbol{\mathsf B}_t.
	\end{align*}
	Finally $\mathbb E\left[\boldsymbol{\mathsf Z}_t\mid\mathcal F^{\,\mathrm{rew}}_{t-1}\right]
	=\mathrm{diag}(\mathsf w_t)^{-1}\mathbb E\left[\boldsymbol{\mathsf X}_t\mid\mathcal F^{\,\mathrm{rew}}_{t-1}\right]
	=
	0$ since $\mathrm{diag}(\mathsf w_t)^{-1}$ is $\mathcal F^{\,\mathrm{rew}}_{t-1}$-measurable.
\end{proof}

\section{The influence-function ascent algorithm}
\label{sec:algorithm}

In this section, we introduce an influence-function ascent procedure described in Algorithm~\ref{alg:if-ascent-truncated}. The algorithm maintains unconstrained dual parameters (logits) $\mathsf h_t\in\mathbb R^K$ which induce a sampling distribution $\mathsf w_t=s(\mathsf h_t)$ on the simplex. At each round an action is drawn according to $\mathsf w_t$, a reward is observed, and a single-sample ascent direction is constructed using an estimator of the influence function associated with the statistical utility.

The resulting update admits two equivalent representations. In simplex coordinates it corresponds to an entropic mirror-ascent step with respect to the negative-entropy mirror map, whose Bregman divergence (see \cite{Bregman1967}) is the Kullback-Leibler (KL) divergence. In dual coordinates it reduces to an additive update in logits, yielding a closed-form multiplicative-weights update in the simplex after applying the softmax map.

To ensure bounded variance, we project the iterate onto the truncated simplex
\[
\Delta_\gamma
=
\left\{w\in\mathbb{R}^K: w_k\geqslant \gamma\, \forall\, k, \sum_{k=1}^K w_k=1\right\},
\]
for a fixed $\gamma\in(0,1/K)$ using the KL projection. Here, $\gamma$ acts as an approximation tolerance that trades off accuracy against the noise variance of the stochastic procedure. This type of constraint can also be viewed as a forced-exploration condition on the simplex. Similar lower bounds on action probabilities are standard in the adversarial bandit and online mirror-descent literature; see e.g., \cite{BubeckCesaBianchi2012,ShalevShwartz2012}.

Recall that the negative-entropy mirror map on $\mathring\Delta$ is given by
\begin{equation}\label{eq:entropy_mirror_map}
	\psi(w)=\sum_{k=1}^K w_k\log w_k,\qquad w\in\mathring{\Delta}.
\end{equation}
Its Bregman divergence is the Kullback-Leibler divergence
\[
D_\psi(u\|w)
=
\psi(u)-\psi(w)-\langle \nabla\psi(w),u-w\rangle
=
\sum_{k=1}^K u_k\log\frac{u_k}{w_k}=D_{\mathrm{KL}}(u\|w).
\]
Moreover,
\(
(\nabla\psi(w))_k=\log w_k+1
\)
so the corresponding dual coordinates are $h = \nabla\psi(w)=\log w+\mathbf 1,$ and hence $w=s(h)$. As usual, $h$ is defined only up to addition of a multiple of $\mathbf 1$ which does not change $s(h)$.

For $\gamma\in(0,1/K)$ consider the associated KL projection operator
\[
\Pi_{\Delta_\gamma}^{\mathrm{KL}}(w)
=
\argmin_{u\in\Delta_\gamma} D_{\mathrm{KL}}(u\|w).
\]
We use the standard entropic mirror-ascent geometry on the simplex, see e.g., \cite{BeckTeboulle2003}. For efficient computation of Bregman projections onto the simplex, see \cite{KricheneKricheneBayen2015}.

\begin{algorithm}
	\caption{Influence-function Mirror Ascent on $\Delta_\gamma$}
	\label{alg:if-ascent-truncated}
	\begin{algorithmic}[1]
		\Require horizon $T$, step sizes $(\eta_t)_{t=1}^N$, floor $\gamma\in(0,1/K)$, initial logits $\mathsf h_1\in\mathbb R^K$
		\State $\mathsf w_1 \gets \Pi_{\Delta_\gamma}^{\mathrm{KL}}(s(\mathsf h_1))\in\Delta_\gamma$
		\For{$t=1,\dots,T$}
		\State Draw action $\mathsf A_t\sim \mathsf w_t$
		\State Observe reward $\mathsf R_t\mid\{\mathsf A_t=k\}\sim P^k$
		\State Construct $\widehat{\mathsf{IF}}_{t-1}$ from $\mathcal F^{\,\mathrm{rew}}_{t-1}$
		\For{$k=1,\dots,K$}
		\State $\widehat{\mathsf H}_{t,k}\gets(\mathds{1}_{\{\mathsf A_t=k\}}-\mathsf w_{t,k})\,\widehat{\mathsf{IF}}_{t-1}(\mathsf R_t)$
		\State $\widehat{\mathsf G}_{t,k}\gets \widehat{\mathsf H}_{t,k}/\mathsf w_{t,k}$ \Comment{$\widehat{\boldsymbol{\mathsf G}}_t=\mathrm{diag}(\mathsf w_t)^{-1}\widehat{\boldsymbol{\mathsf H}}_t$}
		\State $\tilde{\mathsf w}_{t+1,k}\gets \mathsf w_{t,k}\exp(\eta_t\,\widehat{\mathsf G}_{t,k})$
		\EndFor
		\State $\bar{\mathsf w}_{t+1}\gets \tilde{\mathsf w}_{t+1}\Big/\sum_{j=1}^K \tilde{\mathsf w}_{t+1,j}$ \Comment{normalize to $\Delta$}
		\State $\mathsf w_{t+1}\gets \Pi_{\Delta_\gamma}^{\mathrm{KL}}(\bar{\mathsf w}_{t+1})\in\Delta_\gamma$
		\EndFor
	\end{algorithmic}
\end{algorithm}

At round $t\geqslant 1$, given an influence function estimator $\widehat{\mathsf{IF}}_{t-1}$, define the associated simplex-coordinate gradient estimator
\begin{equation}\label{eq:simplex_estimator}
	\widehat{\mathsf G}_{t,k}
	=
	\frac{\widehat{\mathsf H}_{t,k}}{\mathsf w_{t,k}}
	=
	\left(\frac{\mathds{1}_{\{\mathsf A_t=k\}}}{\mathsf w_{t,k}}-1\right)\,
	\widehat{\mathsf{IF}}_{t-1}(\mathsf R_t),
	\qquad k=1,\dots,K,
\end{equation}
where $\widehat{\mathsf H}_{t,k}$ is as in \eqref{eq:def_H}. In vector form 
\(
\widehat{\boldsymbol{\mathsf H}}_t = \mathrm{diag}(\mathsf w_t)\widehat{\boldsymbol{\mathsf G}}_t
\).
By Lemma~\ref{lem:IF-decomposition}, $\widehat{\boldsymbol{\mathsf G}}_t$ is a biased estimator of the centered simplex gradient $\mathbf g_C(\mathsf w_t)$. The estimator has the usual score times signal structure of likelihood-ratio policy-gradient methods, as in \cite{Williams1992}: the centered importance-weighted factor $\frac{\mathds{1}_{\{\mathsf A_t=k\}}}{\mathsf w_{t,k}}-1$ is the simplex-coordinate analogue of a score term, while the scalar signal is given here by the influence-function evaluation rather than by the raw reward; see also \cite{SuttonMcAllesterSinghMansour2000}.

The entropic mirror-ascent subproblem on the simplex is
\begin{equation}\label{eq:mirror_ascent_subproblem}
	\bar{\mathsf w}_{t+1}
	=
	\argmax_{w\in\Delta}
	\left\{
	\eta_t\langle \widehat{\boldsymbol{\mathsf G}}_{t},w\rangle
	-
	D_{\mathrm{KL}}(w\|\mathsf w_t)
	\right\}.
\end{equation}
A standard Lagrange multiplier calculation yields the multiplicative weights update
\begin{equation}\label{eq:mw_update}
	\bar{\mathsf w}_{t+1,k}
	=
	\frac{\mathsf w_{t,k}\exp\left(\eta_t\,\widehat{\mathsf G}_{t,k}\right)}
	{\sum_{j=1}^K \mathsf w_{t,j}\exp\left(\eta_t\,\widehat{\mathsf G}_{t,j}\right)},
	\qquad k=1,\dots,K.
\end{equation}
With the exploration floor $\Delta_\gamma$, we apply the KL projection
\(
\mathsf w_{t+1}=\Pi_{\Delta_\gamma}^{\mathrm{KL}}(\bar{\mathsf w}_{t+1})
\).
Finally, the multiplicative update \eqref{eq:mw_update} admits an equivalent additive form in dual coordinates: for any representative $\mathsf h_t$ such that $\mathsf w_t=s(\mathsf h_t)$, $\mathsf h_{t+1}=\mathsf h_t+\eta_t\,\widehat{\boldsymbol{\mathsf G}}_{t},$ $\bar{\mathsf w}_{t+1}=s(\mathsf h_{t+1}),$ $\mathsf w_{t+1}=\Pi_{\Delta_\gamma}^{\mathrm{KL}}(\bar{\mathsf w}_{t+1}).$ This is the dual coordinate mirror-ascent update associated with the negative-entropy mirror map that motivates Algorithm~\ref{alg:if-ascent-truncated}. 

The proof of the following Lemma is included for the sake of completeness. Closely related calculations can be found in \citet[Lemma 6.10]{Orabona2019} and in the entropic mirror-descent analysis of \citet[Theorem 28.4 and Proposition 28.7]{lattimore_bandit_2020}.

\begin{lemma}
	\label{lem:one_step_mirror}
	Suppose Assumption \ref{ass:A1} holds. Fix $t\geqslant 1$ and a step size $\eta_t>0$.
	Let $\mathsf w_t\in\Delta$ and $\bar{\mathsf w}_{t+1}$ as in \eqref{eq:mirror_ascent_subproblem}. Then, for every $u\in\Delta$
	\begin{equation}
		\label{eq:one_step_mirror_general}
		\langle \widehat{\boldsymbol{\mathsf G}}_t,\,u-\mathsf w_t\rangle
		\leqslant
		\frac{1}{\eta_t}\Big(
		D_{\mathrm{KL}}(u\|\mathsf w_t)-D_{\mathrm{KL}}(u\|\bar{\mathsf w}_{t+1})
		\Big)
		+\frac{\eta_t}{2}\,\|\widehat{\boldsymbol{\mathsf G}}_t\|_\infty^2.
	\end{equation}
\end{lemma}

\begin{proof}
	Define  
	\(
	F(w)=\eta_t\langle \widehat{\boldsymbol{\mathsf G}}_t,w\rangle - D_{\mathrm{KL}}(w\|\mathsf w_t),
	\)
	with $w\in\Delta$.
	By definition $\bar{\mathsf w}_{t+1}$ maximizes $F$ over the convex set $\Delta$ and since $F$ is concave
	\begin{equation}
		\label{eq:opt_condition}
		\langle \nabla F(\bar{\mathsf w}_{t+1}),\,u-\bar{\mathsf w}_{t+1}\rangle \leqslant 0,
	\end{equation}
	for every $u\in\Delta$. Plugging $\nabla F(w)=\eta_t\widehat{\boldsymbol{\mathsf G}}_t-(\nabla\psi(w)-\nabla\psi(\mathsf w_t))$ into \eqref{eq:opt_condition} yields
	\begin{equation}
		\label{eq:key_ineq1}
		\eta_t\langle \widehat{\boldsymbol{\mathsf G}}_t,\,u-\bar{\mathsf w}_{t+1}\rangle
		\leqslant
		\langle \nabla\psi(\bar{\mathsf w}_{t+1})-\nabla\psi(\mathsf w_t),\,u-\bar{\mathsf w}_{t+1}\rangle,
	\end{equation}
	where $\psi$ is the mirror map defined in \eqref{eq:entropy_mirror_map}. 
	
	Apply the three-point Bregman identity (\cite{Bregman1967}) to rewrite the RHS of
	\eqref{eq:key_ineq1}:
	\[
	\langle \nabla\psi(\bar{\mathsf w}_{t+1})-\nabla\psi(\mathsf w_t),\,u-\bar{\mathsf w}_{t+1}\rangle
	=
	D_{\mathrm{KL}}(u\|\mathsf w_t)-D_{\mathrm{KL}}(u\|\bar{\mathsf w}_{t+1})
	-D_{\mathrm{KL}}(\bar{\mathsf w}_{t+1}\|\mathsf w_t).
	\]
	Hence
	\begin{equation}
		\label{eq:key_ineq2}
		\langle \widehat{\boldsymbol{\mathsf G}}_t,\,u-\bar{\mathsf w}_{t+1}\rangle
		\leqslant
		\frac{1}{\eta_t}\Big(
		D_{\mathrm{KL}}(u\|\mathsf w_t)-D_{\mathrm{KL}}(u\|\bar{\mathsf w}_{t+1})
		-D_{\mathrm{KL}}(\bar{\mathsf w}_{t+1}\|\mathsf w_t)
		\Big).
	\end{equation}
	Sum $\langle \widehat{\boldsymbol{\mathsf G}}_t,\bar{\mathsf w}_{t+1}-\mathsf w_t\rangle$ to both sides of \eqref{eq:key_ineq2} to get
	\begin{equation}
		\langle \widehat{\boldsymbol{\mathsf G}}_t,\,u-\mathsf w_t\rangle
		\leqslant
		\frac{1}{\eta_t}\Big(
		D_{\mathrm{KL}}(u\|\mathsf w_t)-D_{\mathrm{KL}}(u\|\bar{\mathsf w}_{t+1})
		\Big)
		-\frac{1}{\eta_t}D_{\mathrm{KL}}(\bar{\mathsf w}_{t+1}\|\mathsf w_t)
		+\langle \widehat{\boldsymbol{\mathsf G}}_t,\bar{\mathsf w}_{t+1}-\mathsf w_t\rangle .
		\label{eq:key_ineq3}
	\end{equation}
	
	Using the fact that $\psi$ is $1$-strongly convex w.r.t.\ $\|\cdot\|_1$ on the simplex (Pinsker's inequality, see for instance \cite{CoverThomas2006,CsiszarKorner2011}) and H\"older-Young's inequality, we get
	\[
	-\frac{1}{\eta_t}D_{\mathrm{KL}}(\bar{\mathsf w}_{t+1}\|\mathsf w_t)
	+\langle \widehat{\boldsymbol{\mathsf G}}_t,\bar{\mathsf w}_{t+1}-\mathsf w_t\rangle
	\leqslant
	\frac{\eta_t}{2}\|\widehat{\boldsymbol{\mathsf G}}_t\|_\infty^2.
	\]
	Plugging this into \eqref{eq:key_ineq3} yields \eqref{eq:one_step_mirror_general}.
\end{proof}

\section{Regret analysis}\label{sec:regret_analysis}

Throughout this section we study Algorithm~\ref{alg:if-ascent-truncated} on the truncated simplex $\Delta_\gamma$. We define the cumulative regret against the optimal utility on $\Delta_\gamma$:
\[
\mathsf{Reg}_\gamma(T)=\sum_{t=1}^T(U_\gamma^\star-U(\mathsf w_t)),
\qquad
U_\gamma^\star\in\max_{w\in\Delta_\gamma}U(w).
\]
In the sequel we denote $w^\star_\gamma\in\Delta_\gamma$ any point such that $U(w_\gamma^\star)=U^\star_\gamma$.

In addition to Assumption \ref{ass:A1}, for the next proposition we will assume the following:
\begin{assumption}\label{ass:A2}
	The function $U:\Delta\to\mathbb R$ is concave. In particular, for all
	$u,w\in\mathring{\Delta}$,
	\[
	U(u)\ \leqslant\ U(w)+\langle \mathbf g_C(w),u-w\rangle .
	\]
\end{assumption}
The decomposition below is the entropic mirror-ascent analogue of the standard one-step regret bound, with an additional term accounting for the bias of the stochastic gradient estimator. See for instance \cite{duchi2010composite}.

\begin{proposition}\label{prop:regret_decomp_concave}
	Suppose Assumptions~\ref{ass:A1} and \ref{ass:A2} hold. Let $\{\mathsf w_t\}_{t\geqslant 1}$ be the resulting sequence of mixture weights obtained by running Algorithm~\ref{alg:if-ascent-truncated}. Let $\mathsf D_t=D_{\mathrm{KL}}(w_\gamma^\star\|\mathsf w_t)$.
	Then, for all $T\geqslant 1$
	\begin{align}
		\mathbb E\left[\mathsf{Reg}_\gamma(T)\right]
		&\leqslant
		\sum_{t=1}^T \frac{1}{\eta_t}\mathbb E\left[\mathsf D_t-\mathsf D_{t+1}\right]
		+\frac{1}{2}\sum_{t=1}^T \eta_t\mathbb E\left[\|\widehat{\boldsymbol{\mathsf G}}_t\|_\infty^2\right]
		+\sum_{t=1}^T \mathbb E\left[\langle \boldsymbol{\mathsf B}_t,\mathsf w_t-w_\gamma^\star\rangle\right].
		\label{eq:regret_decomp_finiteT}
	\end{align}
	
\end{proposition}

\begin{proof}
	Fix $t\geqslant 1$. Assumption~\ref{ass:A2} gives
	\(
	U_\gamma^\star-U(\mathsf w_t)
	\leqslant
	\langle \mathbf g_C(\mathsf w_t),\,w_\gamma^\star-\mathsf w_t\rangle
	\).
	By Lemma \ref{lem:IF-decomposition}  we have
	\[
	\left\langle \mathbf g_C(\mathsf w_t), w_\gamma^\star-\mathsf w_t\right\rangle
	=
	\mathbb E\left[\langle \widehat{\boldsymbol{\mathsf G}}_t, w_\gamma^\star-\mathsf w_t\rangle\mid\mathcal F^{\,\mathrm{rew}}_{t-1}\right]
	-
	\left\langle \boldsymbol{\mathsf B}_t, w_\gamma^\star-\mathsf w_t\right\rangle.
	\]
	Since $\mathsf w_{t+1}=\Pi_{\Delta_\gamma}^{\mathrm{KL}}(\bar{\mathsf w}_{t+1})$, we have
	\(
	D_{\mathrm{KL}}(w_\gamma^\star \| \bar{\mathsf w}_{t+1})
	\geqslant
	D_{\mathrm{KL}}(w_\gamma^\star \| \mathsf w_{t+1})
	=
	\mathsf D_{t+1}
	\).
	Then Lemma \ref{lem:one_step_mirror} yields
	\[
	\langle \widehat{\boldsymbol{\mathsf G}}_t,w_\gamma^\star-\mathsf w_t\rangle
	\leqslant
	\frac{1}{\eta_t}\left(\mathsf D_t-D_{\mathrm{KL}}(w_\gamma^\star \| \bar{\mathsf w}_{t+1})\right)
	+\frac{\eta_t}{2}\|\widehat{\boldsymbol{\mathsf G}}_t\|_\infty^2
	\leqslant
	\frac{1}{\eta_t}(\mathsf D_t-\mathsf D_{t+1})
	+\frac{\eta_t}{2}\|\widehat{\boldsymbol{\mathsf G}}_t\|_\infty^2.
	\]
	Taking the conditional expectations given $\mathcal F^{\,\mathrm{rew}}_{t-1}$, then the total expectations and summing
	from $t=1$ to $T$ gives \eqref{eq:regret_decomp_finiteT}.
\end{proof}

For the results of this section we will use the following quantitative estimate of the KL-diameter of $\Delta_\gamma$. The finiteness is trivial by compactness and continuity, but we emphasize the explicit bound.

\begin{lemma}\label{lem:KL_diameter}
	For $\gamma\in(0,1/K)$,
	\(
	D_{\max}=\sup_{u,w\in\Delta_\gamma}D_{\mathrm{KL}}(u\|w)
	\leqslant \log\big(1/\gamma- (K-1)\big)<\infty.
	\)
\end{lemma}

\begin{proof}
	For any $u,w\in\Delta_\gamma$,
	\(
	D_{\mathrm{KL}}(u\|w)
	\leqslant \sum_{k=1}^K u_k\log u_k - \log \gamma.
	\)
	Since $\log u_k \leqslant \log(\max_j u_j)$ for every $k$, we have
	\[
	\sum_{k=1}^K u_k\log u_k
	\leqslant \sum_{k=1}^K u_k \log\Big(\max_j u_j\Big)
	= \log\Big(\max_j u_j\Big).
	\]
	In $\Delta_\gamma$ we have $\max_j u_j\leqslant 1-(K-1)\gamma$, hence
	\(
	D_{\mathrm{KL}}(u\|w)\leqslant 
	\log\big(1/\gamma- (K-1)\big)\).
\end{proof}

For the following results we assume the standard variance-control requirement on the estimated influence-function signal.

\begin{assumption}\label{ass:A4}
	There exists $V<\infty$ such that
	$\mathbb E\left[\widehat{\mathsf{IF}}_{t-1}(\mathsf R_t)^2\mid\mathcal F^{\,\mathrm{rew}}_{t-1}\right]
	\leqslant V$ a.s.\ for all $t\geqslant 1$.
\end{assumption}

The first result establishes a finite time average regret bound with biased influence function estimates. It gives a $\varepsilon$-guarantee of average regret provided the bias and the learning rate are sufficiently small and the horizon is large enough.

\begin{theorem}\label{thm:avg_regret_eps_bias}
	Suppose Assumptions~\ref{ass:A1} to \ref{ass:A4} hold. Run Algorithm~\ref{alg:if-ascent-truncated} with a constant step size $\eta_t\equiv \eta>0$.
	Define
	\begin{equation*}\label{eq:average_bias}
		\beta_T
		=
		\frac{1}{T}\sum_{t=1}^T \mathbb E\left[\|\boldsymbol{\mathsf B}_t\|_\infty\right].
	\end{equation*}
	Given $\varepsilon>0$, if
	\(
	\beta_T<\varepsilon/4
	\),
	\(
	\eta < (\gamma^2/2V)\varepsilon
	\),
	and
	\(
	T > 4 D_{\max}/\eta\varepsilon,
	\)
	then
	\(
	\mathbb E[\mathsf{Reg}_\gamma(T)]/T
	<
	\varepsilon.
	\)
\end{theorem}

\begin{proof}
	We first show that for all $T\geqslant 1$
	\begin{equation}\label{eq:regret_bound_finiteT}
		\frac{\mathbb E[\mathsf{Reg}_\gamma(T)]}{T}
		\leqslant
		\frac{D_{\max}}{\eta T}
		+
		\frac{\eta}{2}\frac{V}{\gamma^2}
		+
		2\beta_T.
	\end{equation}
	By definition \eqref{eq:simplex_estimator} and the fact that $\mathsf w_t\in\Delta_\gamma$, we have for every $k$
	\[
	|\widehat{\mathsf G}_{t,k}|
	=
	\left|\left(\frac{\mathds{1}_{\{\mathsf A_t=k\}}}{\mathsf w_{t,k}}-1\right)\widehat{\mathsf{IF}}_{t-1}(\mathsf R_t)\right|
	\leqslant
	\frac{1}{\gamma}\,|\widehat{\mathsf{IF}}_{t-1}(\mathsf R_t)|.
	\]
	Hence
	$\|\widehat{\boldsymbol{\mathsf G}}_t\|_\infty^2
	\leqslant (\widehat{\mathsf{IF}}_{t-1}(\mathsf R_t)/\gamma)^2$. 
	Taking conditional expectations and using Assumption~\ref{ass:A4} yields
	\begin{equation}\label{eq:cotahatG}
		\mathbb E\left[\|\widehat{\boldsymbol{\mathsf G}}_t\|_\infty^2\mid\mathcal F^{\,\mathrm{rew}}_{t-1}\right]
		\leqslant
		\frac{V}{\gamma^2}\quad \text{a.s.}\qquad (t\geqslant 1)
	\end{equation}
	and therefore $\mathbb E[\|\widehat{\boldsymbol{\mathsf G}}_t\|_\infty^2]\leqslant V/\gamma^2$. Apply Proposition~\ref{prop:regret_decomp_concave} with the constant step size $\eta_t\equiv \eta$:
	\[
	\mathbb E[\mathsf{Reg}_\gamma(T)]
	\leqslant
	\sum_{t=1}^T \frac{1}{\eta}\,\mathbb E[\mathsf D_t-\mathsf D_{t+1}]
	+\frac12\sum_{t=1}^T \eta\,\mathbb E[\|\widehat{\boldsymbol{\mathsf G}}_t\|_\infty^2]
	+\sum_{t=1}^T \mathbb E[\langle \boldsymbol{\mathsf B}_t,\,\mathsf w_t-w_\gamma^\star\rangle].
	\]
	By \eqref{eq:cotahatG} we have
	\( (1/2)\sum_{t=1}^T \eta\,\mathbb E[\|\widehat{\boldsymbol{\mathsf G}}_t\|_\infty^2]
	\leqslant \eta VT/(2\gamma^2).
	\) Using 
	\(\langle \boldsymbol{\mathsf B}_t,\mathsf w_t-w_\gamma^\star\rangle
	\leqslant
	2\|\boldsymbol{\mathsf B}_t\|_\infty\),  
	taking expectations and summing 
	\begin{equation}\label{eq:bound_beta}
		\sum_{t=1}^T \mathbb E[\langle \boldsymbol{\mathsf B}_t,\mathsf w_t-w_\gamma^\star\rangle]
		\leqslant
		2\sum_{t=1}^T \mathbb E[\|\boldsymbol{\mathsf B}_t\|_\infty]
		=
		2T\beta_T.
	\end{equation}
	Finally since $\mathsf D_{T+1}\geqslant 0$
	\[
	\sum_{t=1}^T \frac{1}{\eta}\mathbb E[\mathsf D_t-\mathsf D_{t+1}]
	=
	\frac{1}{\eta}\mathbb E[\mathsf D_1-\mathsf D_{T+1}]
	\leqslant
	\frac{D_{\max}}{\eta}.
	\]
	Combining these bounds yields \eqref{eq:regret_bound_finiteT}. Plugging in the inequalities for $\beta_T$, $\eta$ and $T$ in terms of $\varepsilon$ gives the desired result.
\end{proof}

The second result gives a horizon-free regret bound when the step sizes are chosen as $\eta_t=\eta_0/\sqrt{t}$. The expected regret is controlled by the usual $\sqrt{T}$ term together with a cumulative bias contribution. In particular, when the bias contribution is sublinear, we obtain sublinear regret rate. See Section \ref{sec:bias_control}.

\begin{theorem}
	\label{thm:regret_concave_anytime_Dmax}
	Suppose Assumptions~\ref{ass:A1} to \ref{ass:A4} hold. Let $\eta_t=\eta_0/\sqrt{t}$.
	Then for all $T\geqslant 1$
	\begin{equation}\label{eq:regret_bound_anytime_Dmax_simplified}
		\mathbb E[\mathsf{Reg}_\gamma(T)]
		\leqslant
		\left(\frac{D_{\max}}{\eta_0}+\frac{\eta_0 V}{\gamma^2}\right)\sqrt{T}
		+
		2\sum_{t=1}^T \mathbb E\left[\|\boldsymbol{\mathsf B}_t\|_\infty\right].
	\end{equation}
\end{theorem}

\begin{proof}
	We first show the following general inequality for non-increasing $\eta_t$:
	\begin{equation}\label{eq:regret_bound_anytime_Dmax_general}
		\mathbb E[\mathsf{Reg}_\gamma(T)]
		\leqslant
		\frac{D_{\max}}{\eta_T}
		+
		\frac{V}{2\gamma^2}\sum_{t=1}^T \eta_t
		+
		2\sum_{t=1}^T \mathbb E\left[\|\boldsymbol{\mathsf B}_t\|_\infty\right].
	\end{equation}
	Since $a_t=1/\eta_t$ is non-decreasing the discrete summation by parts identity gives
	\[
	\sum_{t=1}^T a_t(\mathsf{D}_t-\mathsf{D}_{t+1})
	=
	a_1\mathsf{D}_1+\sum_{t=2}^T (a_t-a_{t-1})\mathsf{D}_t-a_T\mathsf{D}_{T+1}.
	\]
	Using $0\leqslant \mathsf{D}_t\leqslant D_{\max}$ and $a_t-a_{t-1}\ge0$ yields
	\[
	\sum_{t=1}^T a_t\left(\mathsf{D}_t-\mathsf{D}_{t+1}\right)
	\leqslant
	a_1 D_{\max}+\sum_{t=2}^T (a_t-a_{t-1})D_{\max}
	=
	a_T D_{\max}
	=
	\frac{D_{\max}}{\eta_T}.
	\]
	Therefore, we can bound the first term of the RHS of \eqref{eq:regret_decomp_finiteT} by
	\(D_{\max}/\eta_T.\)
	By \eqref{eq:cotahatG} the second term of the RHS of \eqref{eq:regret_decomp_finiteT} can be bounded by  
	\(V/(2\gamma^2)\sum_{t=1}^T \eta_t.
	\)
	For the third term of the RHS of \eqref{eq:regret_decomp_finiteT} we use the same bound as in \eqref{eq:bound_beta}. Combining these three bounds gives \eqref{eq:regret_bound_anytime_Dmax_general}.
	
	Finally \eqref{eq:regret_bound_anytime_Dmax_simplified} follows by taking $\eta_t=\eta_0/\sqrt{t}$, using $\eta_T=\eta_0/\sqrt{T}$ and $\sum_{t=1}^T t^{-1/2}\leqslant 2\sqrt{T}$.
\end{proof}

\begin{corollary}\label{cor:optima_eta_0}
	Optimizing \eqref{eq:regret_bound_anytime_Dmax_simplified} over $\eta_0$ yields $\eta_0=\sqrt{D_{\max}\gamma^2/V}$ and hence
	\[
	\mathbb E[\mathsf{Reg}_\gamma(T)]
	\leqslant
	2\left(\frac{V}{\gamma^2}\,D_{\max}\right)^{1/2} \sqrt{T}
	+
	2\sum_{t=1}^T \mathbb E\left[\|\boldsymbol{\mathsf B}_t\|_\infty\right].
	\]
\end{corollary}

\begin{corollary}
	\label{cor:avg_iterate_regret}
	Let $\bar{\mathsf w}_T=\frac{1}{T}\sum_{t=1}^T \mathsf w_t$. Then
	\[
	\mathbb E \left[U_\gamma^\star-U(\bar{\mathsf w}_T)\right]
	\leqslant
	2\left(\frac{V}{\gamma^2}\,D_{\max}\right)^{1/2} \sqrt{T}
	+
	2\sum_{t=1}^T \mathbb E\left[\|\boldsymbol{\mathsf B}_t\|_\infty\right].
	\]
\end{corollary}

\begin{proof}
	Follows from Corollary \ref{cor:optima_eta_0} and concavity of $U$.
\end{proof}

\section{Bias control for plug-in influence functions}\label{sec:bias_control}

This section provides a compact reformulation and convenient upper bounds for the bias term
\(
\sum_{t=1}^T \mathbb E\left[\|\boldsymbol{\mathsf B}_t\|_\infty\right]
\)
appearing in Theorems~\ref{thm:avg_regret_eps_bias} and \ref{thm:regret_concave_anytime_Dmax}, and Corollaries \ref{cor:optima_eta_0} and \ref{cor:avg_iterate_regret}.

Define the error in the estimation of influence-function at time $t$ by
$\mathsf E_t(r)
=
\widehat{\mathsf{IF}}_{t-1}(r)-\IF\left[P^{\mathsf w_t}\right](r),$ $r\in\mathbb R.$

\begin{lemma}\label{lem:bias_centered_and_bounds}
	Suppose Assumptions~\ref{ass:A1} to \ref{ass:A4} hold. If in addition $\mathsf E_t\in L^2(P^k)$ for every $k\in\{1,\ldots,K\}$ then 
	\begin{equation}\label{eq:bias_bound_L2}
		\|\boldsymbol{\mathsf B}_t\|_\infty
		\leqslant
		2\max_{1\leqslant k\leqslant K}
		\left\|\mathsf E_t(\mathsf R)\right\|_{L^2(P^k)}.
	\end{equation}
\end{lemma}

\begin{proof}
	By definition of $\mathsf{E}_t$ we have $\mathsf B_{t,k}
	=
	\mathbb E_{P^k}\left[\mathsf E_t(\mathsf R)\right]-\mathbb E_{P^{\mathsf w_t}}\left[\mathsf E_t(\mathsf R)\right]$. Also
	\(
	\left|\mathbb E_{P^{\mathsf w_t}}[\mathsf E_t(\mathsf R)]\right|
	\leqslant
	\max_j \left|\mathbb E_{P^j}[\mathsf E_t(\mathsf R)]\right|.
	\)
	Therefore $|\mathsf B_{t,k}|
	\leqslant
	2\max_{1\leqslant j\leqslant K}\left|\mathbb E_{P^j}\left[\mathsf E_t(\mathsf R)\right]\right|$ 
	which by Cauchy-Schwarz inequality implies \eqref{eq:bias_bound_L2}. 
\end{proof}

Lemma~\ref{lem:bias_centered_and_bounds} reduces the control of
\(
\sum_{t=1}^T \mathbb E\left[\|\boldsymbol{\mathsf B}_t\|_\infty\right]
\)
to bounding the arm-wise $L^2$ error of the influence-function estimator. We now give a simple set of sufficient conditions together with a concrete
plug-in construction, under which the cumulative bias grows at most on the order of $\sqrt{T}$.

For each arm $k$ let $\mathsf P^{k}_{t-1}$ be as in \eqref{eq:arm_empirical}. Given the current sampling distribution $\mathsf w_t\in\Delta_\gamma$ define the empirical mixture
$\widehat{\mathsf P}^{\mathsf w_t}
=
\sum_{k=1}^K \mathsf w_{t,k}\mathsf P^{k}_{t-1}.$ 
Consider the plug-in estimator of the influence function
\begin{equation}\label{eq:plugin_IF}
	\widehat{\mathsf{IF}}_{t-1}=
	\IF\left[\widehat{\mathsf P}^{\mathsf w_t}\right].
\end{equation}

We will need the following additional assumptions:

\begin{assumption}\label{ass:A5}
	There exists $p\in[1,\infty)$ such that $d(\cdot,\cdot)^p$ is jointly convex.
\end{assumption}

\begin{assumption}\label{ass:A6}
	There exists a constant $L<\infty$ such that for every $t\geqslant 1$ and arm $k$
	\begin{equation*}\label{eq:IF_stability_new}
		\left\|\mathsf E_t(\mathsf R)
		\right\|_{L^2(P^k)}
		\leqslant
		Ld\left(\widehat{\mathsf P}^{\mathsf w_t},P^{\mathsf w_t}\right)
		\quad \text{a.s.}
	\end{equation*}
\end{assumption}

\begin{assumption}\label{ass:A7}
	There exists a constant $C_d<\infty$ such that for every arm $k$ and every $n\geqslant 1$, if $\mathsf Q_{k,n}$ is the empirical distribution of $n$ iid samples from $P^k$ then
	\begin{equation*}\label{eq:d_empirical_rate}
		\mathbb E\left[d\left(\mathsf Q_{k,n},P^k\right)\right]\leqslant \frac{C_d}{\sqrt{n}}.
	\end{equation*}    
\end{assumption}

\begin{proposition}\label{prop:bias_sqrtT}
	Suppose Assumptions~\ref{ass:A1} to \ref{ass:A7} hold. Run Algorithm~\ref{alg:if-ascent-truncated} and let $\widehat{\mathsf{IF}}_{t-1}$ be as in \eqref{eq:plugin_IF}. Then for all $t\geqslant 2$,
	\begin{equation}\label{eq:Bt_pointwise_bound_new}
		\mathbb E\left[\|\boldsymbol{\mathsf B}_t\|_\infty\right]
		\leqslant
		\frac{4LC_d}{\sqrt{\gamma (t-1)}}
		+
		2LC_d\exp\left(-\frac{\gamma^2}{8}(t-1)\right)
		\leqslant
		\frac{8LC_d}{\sqrt{\gamma}}\frac{1}{\sqrt{t}}.
	\end{equation}
\end{proposition}

\begin{proof}
	Fix $t\geqslant 2$. By \eqref{eq:bias_bound_L2} and Assumption~\ref{ass:A6} we have
	\(
	\|\boldsymbol{\mathsf B}_t\|_\infty
	\leqslant
	2 L d\left(\widehat{\mathsf P}^{\mathsf w_t},P^{\mathsf w_t}\right)
	\)
	a.s..
	Taking expectations yields
	\begin{equation}\label{eq:Bt_expectation_reduction}
		\mathbb E\left[\|\boldsymbol{\mathsf B}_t\|_\infty\right]
		\leqslant
		2L\mathbb E\left[d\left(\widehat{\mathsf P}^{\mathsf w_t},P^{\mathsf w_t}\right)\right].
	\end{equation}
	
	We now bound
	$\mathbb E\left[d\left(\widehat{\mathsf P}^{\mathsf w_t},P^{\mathsf w_t}\right)\right]$.
	By Assumption \ref{ass:A5} we have
	\begin{equation*}
		d\left(\widehat{\mathsf P}^{\mathsf w_t},P^{\mathsf w_t}\right)^p
		=
		d\left(\sum_{k=1}^K \mathsf w_{t,k}\mathsf P^{k}_{t-1}, \sum_{k=1}^K \mathsf w_{t,k}P^k\right)^p
		\leqslant
		\sum_{k=1}^K \mathsf w_{t,k} d\left(\mathsf P^{k}_{t-1},P^k\right)^p.
	\end{equation*}
	Using $\sum_k \mathsf w_{t,k}=1$ and taking expectations yields
	\begin{equation}\label{eq:mix_expect_reduce}
		\mathbb E\left[d\left(\widehat{\mathsf P}^{\mathsf w_t},P^{\mathsf w_t}\right)\right]
		\leqslant
		\max_{1\leqslant k\leqslant K}\mathbb E\left[d\left(\mathsf P^{k}_{t-1},P^k\right)\right].
	\end{equation}
	
	Fix $k$. Given $\mathsf N_{t-1,k}=n$, the distribution $\mathsf P^{k}_{t-1}$ is the empirical law of $n$ iid samples from $P^k$, hence by Assumption~\ref{ass:A7} we have \(\mathbb E[d(\mathsf P^{k}_{t-1},P^k)|\mathsf N_{t-1,k}=n]\leqslant C_d/\sqrt{n}.\)
	Therefore
	\begin{equation}\label{eq:reduce_to_inverse_sqrt_N}
		\mathbb E\left[d\left(\mathsf P^{k}_{t-1},P^k\right)\right]
		\leqslant
		C_d\mathbb E\left[\frac{1}{\sqrt{\mathsf N_{t-1,k}}}\right].
	\end{equation}
	
	Since the algorithm runs on $\Delta_\gamma$ we have $\mathbb E[\mathds 1_{\{\mathsf A_s=k\}}\mid\mathcal F^{\,\mathrm{rew}}_{s-1}]=\mathsf w_{s,k}\geqslant\gamma$ for all $s$. Let $\mathsf Y_s=\mathds 1_{\{\mathsf A_s=k\}}-\mathsf w_{s,k}$ so that $\left\{\sum_{s=1}^n \mathsf Y_s\right\}_{n\geqslant 1}$ is a martingale with increments in $[-1,1]$ w.r.t $\{\mathcal F^{\,\mathrm{rew}}_{n}\}_{n\geqslant 1}$. Since $\mathsf N_{t-1,k}=\sum_{s=1}^{t-1}\mathsf w_{s,k}+\sum_{s=1}^{t-1}\mathsf Y_s\geqslant \gamma(t-1)+\sum_{s=1}^{t-1}\mathsf Y_s$, Azuma's inequality yields
	\[
	\mathbb P\left(\mathsf N_{t-1,k}\leqslant \frac{\gamma}{2}(t-1)\right)
	\leqslant
	\mathbb P\left(\sum_{s=1}^{t-1}\mathsf Y_s\leqslant -\frac{\gamma}{2}(t-1)\right)
	\leqslant
	\exp\left(-\frac{\gamma^2}{8}(t-1)\right).
	\]
	Consequently
	\[
	\mathbb E\left[\frac{1}{\sqrt{\mathsf N_{t-1,k}}}\right]
	\leqslant
	\frac{1}{\sqrt{\frac{\gamma}{2}(t-1)}} +\mathbb P\left(\mathsf N_{t-1,k}\leqslant \frac{\gamma}{2}(t-1)\right)
	\leqslant
	\sqrt{\frac{2}{\gamma(t-1)}}+\exp\left(-\frac{\gamma^2}{8}(t-1)\right).
	\]
	Plugging this into \eqref{eq:reduce_to_inverse_sqrt_N} and then into \eqref{eq:mix_expect_reduce} gives
	\begin{equation}\label{eq:mixture_rate_derived}
		\mathbb E\left[d\left(\widehat{\mathsf P}^{\mathsf w_t},P^{\mathsf w_t}\right)\right]
		\leqslant
		\frac{\sqrt{2}C_d}{\sqrt{\gamma (t-1)}} + C_d\exp\left(-\frac{\gamma^2}{8}(t-1)\right).
	\end{equation}
	Combining \eqref{eq:Bt_expectation_reduction} with \eqref{eq:mixture_rate_derived} yields the first bound in \eqref{eq:Bt_pointwise_bound_new} from which the second one follows (for $t\geqslant 2$).
\end{proof}

\begin{corollary}
	For every $T\geqslant 2$ we have
	\(
	\sum_{t=1}^T \mathbb E\left[\|\boldsymbol{\mathsf B}_t\|_\infty\right]
	\leqslant 16LC_d\sqrt{T}/\sqrt{\gamma}.
	\)
\end{corollary}

\section{Applications}\label{sec:applications}

Throughout this section we assume that rewards are uniformly bounded: there exists a constant \(r_{\max}<\infty\) such that \(P^k([-r_{\max},r_{\max}])=1\) for all \(k\in\{1,\dots,K\}\). We restrict the space of probability distributions to \(\mathcal P([-r_{\max},r_{\max}])\). We view the bounded support assumption not as a restriction on the underlying raw outcome, but rather on the effective reward signal optimized by the learner. In many applications this signal is naturally bounded by design (e.g.\ it is a score, a normalized metric, or a clipped utility). This allows us to develop a unified analysis for several distributional utilities without introducing technical assumptions on tails that are secondary to the main ideas of the paper.

\subsection{Variance utility}\label{sec:variance_utility}

A case of interest is $\mathfrak U(P)=\mathbb V(\mathsf R)$, where $\mathsf R\sim P$. For this utility we take \(d = W_1\), the Wasserstein-1 distance on \(\mathcal P([-r_{\max},r_{\max}])\). The distance $W_1(P,Q)$ can be computed as the $L^1$ distance on $\mathbb R$ of the CDF's of $P$ and $Q$:
\begin{equation}\label{eq:L1normW1}
	W_1(P,Q)=\int_{-r_{\max}}^{r_{\max}} |P(r)-Q(r)|dr.
\end{equation}
It induces the weak topology on \(\mathcal P([-r_{\max},r_{\max}])\).

\begin{proposition}\label{prop:metric_continuity_variance}
	The distance \(d\) satisfies \eqref{eq:dist_continuity_1}, the hypothesis of Proposition \ref{prop:mixture_continuity} (in particular, it satisfies \eqref{eq:dist_continuity_2} and Assumption \ref{ass:A5}), and \(\mathfrak U\) is Lipschitz continuous with respect to \(d\).
\end{proposition}
\begin{proof}
	Since all measures are supported on the compact interval \([-r_{\max},r_{\max}]\), weak convergence is equivalent to convergence in \(W_1\); see for example \cite[Theorem 5.10]{santambrogio2015}. Therefore, from \(\mathsf P_n \Rightarrow P\) a.s.\ we obtain \(W_1(\mathsf P_n,P)\to 0\) a.s.. This gives \eqref{eq:dist_continuity_1}. Remark \ref{rem:wasserstein_mixtures} shows that $d$ satisfies the hypothesis of Proposition \ref{prop:mixture_continuity}.
	
	Finally, by Kantorovich-Rubinstein (see Theorem 5.10 in \cite{Villani2009}) duality on \([-r_{\max},r_{\max}]\) we have
	for any \(P,Q\)
	\[
	\left|\mathbb{E}_P[\mathsf R]-\mathbb{E}_Q[\mathsf R]\right|\leqslant W_1(P,Q)\text{ and } \left|\mathbb{E}_P\left[\mathsf R^2\right]-\mathbb{E}_Q\left[\mathsf R^2\right]\right|\leqslant 2r_{\max} W_1(P,Q).
	\]
	Since \(\mathbb{V}_P(\mathsf R)=\mathbb{E}_P[\mathsf R^2]-\mathbb{E}_P[\mathsf R]^2\) and \(|\mathbb{E}_P[\mathsf R]|\leqslant r_{\max}\) we have that \(\mathfrak U\) is Lipschitz continuous w.r.t.\ \(d=W_1\).
\end{proof}

\begin{proposition}\label{prop:IF_variance}
	For \(P\in\mathcal P([-r_{\max},r_{\max}])\) set $\mu_P=\mathbb{E}_{P}[\mathsf R]$ and
	$\sigma_P^2=\mathbb{V}_{P}(\mathsf R).$ 
	Then Assumption~\ref{ass:A1} holds with the influence function
	\(
	\IF[P](r)
	=
	(r-\mu_P)^2-\sigma_P^2
	\).
\end{proposition}
\begin{proof}
	This is well known, see for example \cite[Section 2.1 Example (ii)]{hampel1974}.
\end{proof}

\begin{proposition}\label{prop:concavity_variance}
	The function \(U(w)=\mathfrak U\left(P^w\right)\) satisfies Assumption~\ref{ass:A2}.
\end{proposition}
\begin{proof}
	Let \(\mu_k=\mathbb{E}_{P^k}[\mathsf R]\) and \(\sigma_k^2=\mathbb{V}_{P^k}(\mathsf R)\).
	Since
	\[
	U(w)=\sum_{k=1}^K w_k\left(\sigma_k^2+\mu_k^2\right) - \left(\sum_{k=1}^K w_k\mu_k\right)^2,
	\]
	it is clear that $U$ is concave and differentiable. Hence Assumption~\ref{ass:A2} holds.
\end{proof}

\begin{proposition}\label{prop:A4_variance}
	Assumption~\ref{ass:A4} holds with constant $V=16r_{\max}^4$.
\end{proposition}
\begin{proof}
	For any probability \(P\) supported on \([-r_{\max},r_{\max}]\) we have that its mean and variance satisfy \(|\mu|\leqslant r_{\max}\) and \(\sigma^2\leqslant r_{\max}^2\) respectively.
	Hence for all \(r\in[-r_{\max},r_{\max}]\) we have $\big|\IF[P](r)\big|=\big|(r-\mu)^2-\sigma^2\big|\leqslant 4 r_{\max}^2$
	so \(\widehat{\mathsf{IF}}_{t-1}(\mathsf R_t)^2\leqslant 16r_{\max}^4\) a.s..
\end{proof}

\begin{proposition}\label{prop:A5_variance}
	Assumptions \ref{ass:A6} and \ref{ass:A7} hold with the constants $L=8r_{\max}$ and $C_d=r_{\max}$.
\end{proposition}
\begin{proof}
	Let \(P,Q\) be supported on \([-r_{\max},r_{\max}]\) and write \(\mu_P,\sigma_P^2\) and \(\mu_Q,\sigma_Q^2\) for their mean and variance respectively.
	Using \(|r|\leqslant r_{\max}\), \(|\mu_P|,|\mu_Q|\leqslant r_{\max}\), we obtain
	\[
	\big|(r-\mu_P)^2-(r-\mu_Q)^2\big|
	=
	|(\mu_Q-\mu_P)(2r-\mu_P-\mu_Q)|
	\leqslant 4r_{\max}\,|\mu_P-\mu_Q|.
	\]
	As in Proposition \ref{prop:metric_continuity_variance} we have $|\mu_P-\mu_Q|\leqslant W_1(P,Q)$
	and \(|\sigma_P^2-\sigma_Q^2|\leqslant 4r_{\max} W_1(P,Q)\). Combining yields \(|\IF[P](r)-\IF[Q](r)|\leqslant 8r_{\max}W_1(P,Q)\) for all $r\in[-r_{\max},r_{\max}]$. Hence for every arm \(k\)
	\[
	\left\|\IF[P](\mathsf R)-\IF[Q](\mathsf R)\right\|_{L^2(P^k)}
	\leqslant \left\|\IF[P]-\IF[Q]\right\|_\infty
	\leqslant 8r_{\max} d(P,Q).
	\]
	This gives Assumption~\ref{ass:A6} with \(L=8r_{\max}\).
	
	For Assumption~\ref{ass:A7}, let \(\mathsf Q_{k,n}\) be the empirical law of \(n\) iid samples from \(P^k\). Since for each $r$ we have \(\mathbb{E}\left|\mathsf Q_{k,n}(r)-P(r)\right|\leqslant 1/(2\sqrt n)\), from \eqref{eq:L1normW1} we get
	\[
	\mathbb{E}\left[W_1\left(\mathsf Q_{k,n},P^k\right)\right]
	\leqslant \int_{-r_{\max}}^{r_{\max}} \frac{1}{2\sqrt n}dx
	= \frac{r_{\max}}{\sqrt n}.
	\]
	Then Assumption~\ref{ass:A7} holds with \(C_d=r_{\max}\).
\end{proof}

\subsection{Structure of the maximizers of the variance utility}

Let \(m_{2,k} = \mathbb E_{P^k}\left[\mathsf R^2\right]\) be the second moment of the arm distribution. We can rewrite the utility as $U(w)=m_2^\top w-(\mu^\top w)^2.$ The gradient and Hessian are $\nabla U(w)=m_2-2(\mu^\top w)\mu,$ $\nabla^2 U(w)=-2\mu\mu^\top .$ Therefore, Hessian has rank $1$ and $U$ is not strictly concave in general. In particular,
for any direction $v$ such that $\mu^\top v=0$ we have \(U(w+tv)=U(w)+t\,m_2^\top v,\)
so $U$ is affine in directions orthogonal to $\mu$.

Consider first the maximization of $U$ across the entire simplex $\Delta$. Introduce the multipliers
\[
\lambda\text{ for }\mathbf 1^\top w-1=0, \text{ and } \nu_k\geqslant 0\text{ for }w_k\geqslant 0,\quad k=1,\ldots,K
\]
and the Lagrangian function
\(
\mathcal L(w,\lambda,\nu)=U(w)-\lambda(\mathbf 1^\top w-1) - \nu^\top w
\). Stationarity gives
\[m_{2,k}-2(\mu^\top w^\star)\mu_k=\lambda+\nu_k,\quad k=1,\ldots,K\]
with the complementary conditions $\nu_k w_k^\star=0$ for all $k=1,\ldots,K$.
So, if $w_k^\star>0$ we have
\(
m_{2,k}-2(\mu^\top w^\star)\mu_k=\lambda
\).
In particular, if $w^\star\in\mathring{\Delta}$ then for some scalar $\lambda$ we have
\begin{equation}\label{eq:interior_condition}
	m_2=\lambda \mathbf 1 + 2 (\mu^\top w^\star) \mu .
\end{equation}
Thus, an interior maximizer exists only if the vector of second moments lies in the
2-dimensional affine subspace spanned by $\mathbf 1$ and $\mu$. This condition is restrictive and generically false. Consequently, for generic moment vectors $(\mu,m_2)$ the maximizer of $U$ lies on the boundary of the simplex.

More generally, taking the vector whose entries are the positive entries of $w^\star$ the same conclusion as in \eqref{eq:interior_condition} holds. In particular, all active indices $k$ must satisfy \(m_{2,k}=\lambda +2(\mu^\top w^\star)\mu_k\)
which means that the points $(\mu_k,m_{2,k})\in\mathbb R^2$ corresponding to positive coordinates lie on the same affine line. Hence, unless several arms share this relation, at most two coordinates of $w^\star$ can be positive.

The same analysis can be done with the constraint $w\in \Delta_\gamma$. Thus, for generic reward distributions the maximizer of $U$ is typically unique with only two active arms, with the remaining weights equal to the lower bound $\gamma$.

\subsection{Wasserstein utility} \label{sec:wasserstein_utility}

Let $Q\in \mathcal P([-r_{\max},r_{\max}])$ be a fixed reference distribution. We assume that \(Q\) admits a density \(q\) on \([-r_{\max},r_{\max}]\) such that $\inf_{r\in[-r_{\max},r_{\max}]} q(r) \geqslant m_Q$
for some positive constant $m_Q$, and that each arm distribution \(P^k\) admits a density \(p_k\) on \([-r_{\max},r_{\max}]\) such that $\inf_{r\in[-r_{\max},r_{\max}]} p_k(r) \geqslant m$ for some positive constant \(m\). Let \(d=W_2\) be the Wasserstein-\(2\) metric and consider the utility functional
\(
\mathfrak U(P)=-W_2^2(P,Q)
\).
As before, by compactness the induced topology on \(\mathcal P([-r_{\max},r_{\max}])\) coincides with the weak convergence of distributions.

\begin{proposition}\label{prop:metric_continuity_wasserstein}
	The distance \(d\) satisfies \eqref{eq:dist_continuity_1}, the hypothesis of Proposition \ref{prop:mixture_continuity} (in particular, it satisfies \eqref{eq:dist_continuity_2} and Assumption \ref{ass:A5}), and \(\mathfrak U\) is Lipschitz continuous with respect to \(d\).
\end{proposition}

\begin{proof}
	Assumption \eqref{eq:dist_continuity_1} and the hypothesis of Proposition \ref{prop:mixture_continuity} follow exactly as in the proof of Proposition \ref{prop:metric_continuity_variance}.
	
	Finally \(\mathfrak U(P)=-W_2^2(P,Q)\) is Lipschitz continuous w.r.t.\ \(W_2\) since
	\[
	\left|W_2^2(P,Q)-W_2^2(P',Q)\right|
	\leqslant \left(W_2(P,Q)+W_2(P',Q)\right)W_2(P,P')
	\leqslant 4r_{\max} W_2(P,P'),
	\]
	because \(W_2(\cdot,\cdot)\leqslant 2r_{\max}\) on \([-r_{\max},r_{\max}]\).
\end{proof}

For \(P\in \mathcal P([-r_{\max},r_{\max}])\) consider the monotone optimal transport map
\begin{equation}\label{eq:Tw_def}
	T(r)=Q^{-1}(P(r)),\qquad r\in[-r_{\max},r_{\max}].
\end{equation}
A (unique up to additive constants) Kantorovich potential for the optimal transport problem from $P$ to $Q$ is then given by 
\begin{equation}\label{eq:phiw_def}
	\phi(r)=\int_0^r \left[s-T(s)\right]ds.
\end{equation}

\begin{proposition}\label{prop:A1_wasserstein}
	Assumption~\ref{ass:A1} holds with the influence function
	\begin{equation}\label{eq:IF_w2}
		\IF[P](r)
		=
		-\phi(r)+\mathbb E_{P}[\phi(\mathsf R)].
	\end{equation}    
\end{proposition}

\begin{proof}
	This is shown in \cite[Lemma 21]{manole2024plugin}. See also \cite[Proposition 7.17]{santambrogio2015} or \cite[Theorem 3.1.14]{panaretos2020invitation}.
\end{proof}
\begin{proposition}\label{prop:A2_wasserstein}
	The function \(U\) satisfies Assumption \ref{ass:A2}.
\end{proposition}
\begin{proof}
	Follows directly from the convexity of \(P\mapsto W_2^2(P,Q)\).
\end{proof}
\begin{proposition} \label{prop:A5_wasserstein}
	Assumptions~\ref{ass:A6} and \ref{ass:A7} hold with $L=2(r_{\max}+1/m_Q)$ and $C_d=2/m$. 
\end{proposition}

\begin{proof}
	We prove Assumption~\ref{ass:A6}.
	Let $P$, $P'$ and $Q$ be supported on $[-r_{\max},r_{\max}]$ with densities $p(r),p'(r)\geqslant m$ and $q(r)\geqslant m_Q$.
	Consider the monotone optimal transport maps $T(r)=Q^{-1}(P(r)),$ $T'(r)=Q^{-1}(P'(r))$, $r\in[-r_{\max},r_{\max}]$ as in \eqref{eq:Tw_def}. Let $\phi,\phi'$ be the corresponding Kantorovich potentials defined by \eqref{eq:phiw_def}. Then 
	\begin{equation}\label{eq:phi_sup_bound_W1}
		\|\phi'-\phi\|_\infty
		\leqslant
		\int_{-r_{\max}}^{r_{\max}}|T(s)-T'(s)|ds.
	\end{equation}
	Since $q(r)\geqslant m_Q$ the quantile map $Q^{-1}$ is $1/m_Q$-Lipschitz and thus
	\[
	|T(s)-T'(s)|
	\leqslant
	\frac{1}{m_Q}|P(s)-P'(s)|.
	\]
	Integrating over $[-r_{\max},r_{\max}]$ and using \eqref{eq:phi_sup_bound_W1} and \eqref{eq:L1normW1} gives
	\begin{equation}\label{eq:phi_sup_bound_W2}
		\|\phi'-\phi\|_\infty
		\leqslant
		\frac{1}{m_Q}W_1(P,P')
		\leqslant
		\frac{1}{m_Q}W_2(P,P').
	\end{equation}
	Analogously to \eqref{eq:IF_w2}, consider the error function
	\[
	E(r)
	=
	-\left(\phi'(r)-\phi(r)\right) + \left(\mathbb E_{P'}\left[\phi'(\mathsf R)\right]-\mathbb E_{P}\left[\phi\left(\mathsf R\right)\right]\right).
	\]
	Adding and subtracting \(\mathbb E_{P}\left[\phi'\left(\mathsf R\right)\right]\) yields
	\[
	|E(r)|
	\leqslant
	\left|\phi'(r)-\phi(r)\right|
	+
	\left|\mathbb E_{P}\left[\phi'\left(\mathsf R\right)-\phi(\mathsf R)\right]\right|
	+
	\left|\mathbb E_{P'}\left[\phi'\left(\mathsf R\right)\right]-\mathbb E_{P}\left[\phi'\left(\mathsf R\right)\right]\right|.
	\]
	The first two terms are each bounded by $\|\phi'-\phi\|_\infty$.
	
	For the last term, since $\phi'$ is $2r_{\max}$-Lipschitz on $[-r_{\max},r_{\max}]$, using the
	Kantorovich-Rubinstein duality we get
	\[
	\left|\mathbb E_{P'}\left[\phi'\left(\mathsf R\right)\right]-\mathbb E_{P}\left[\phi'\left(\mathsf R\right)\right]\right|
	\leqslant
	2r_{\max} W_1(P,P')
	\leqslant
	2r_{\max} W_2(P,P').
	\]
	Combining this with \eqref{eq:phi_sup_bound_W2} gives
	$$\|E\|_\infty \leqslant 2\|\phi'-\phi\|_\infty + 2r_{\max} W_2(P,P')\leqslant 2\left(r_{\max}+\frac{1}{m_Q}\right) W_2(P,P').$$
	Therefore, for every $k$ we have
	\(
	\|E(\mathsf R)\|_{L^2(P^k)}
	\leqslant
	\|E\|_\infty
	\leqslant
	L W_2(P,P')
	\)
	with \(L=2\left(r_{\max}+\frac{1}{m_Q}\right)\), which proves Assumption~\ref{ass:A6}.
	
	We now prove Assumption~\ref{ass:A7}.
	Let $\mathsf Q_{k,n}$ be the empirical law of $n$ iid samples from $P^k$.
	Since $p_k(r)\geqslant m$ the quantile map is $1/m$-Lipschitz and therefore
	\(
	\|Q_{k,n}^{-1}-(P^k)^{-1}\|_\infty
	\leqslant (1/m) \|Q_{k,n}-P^k\|_\infty
	\).
	Hence
	\(
	\mathbb E\left[W_2(\mathsf Q_{k,n},P^k)\right]
	\leqslant (1/m) \mathbb E\left[\|Q_{k,n}-P^k\|_\infty\right]
	\).
	As a consequence of Dvoretzky-Kiefer-Wolfowitz inequality  (see \cite{DvoretzkyKieferWolfowitz1956,massart1990tight}) we have
	\(\mathbb E\left[W_2(\mathsf Q_{k,n},P^k)\right]\leqslant (2/m) n^{-1/2}\) so Assumption~\ref{ass:A7} holds with $C_d=2/m$.
\end{proof}

\begin{remark}
	The lower bound assumption \(p_k\ge m>0\) is used only to verify Assumption~\ref{ass:A7} with the rate $n^{-1/2}$ for
	\(
	\mathbb E\left[W_2(\mathsf Q_{k,n},P^k)\right]
	\).
	If this assumption is removed, one still obtains a weaker version of Assumption~\ref{ass:A7} on bounded support by comparing \(W_2\) and \(W_1\). Indeed, since the rewards are supported on \([-r_{\max},r_{\max}]\), we have
	\(
	W_2\le (2r_{\max})^{1/2}W_1^{1/2}
	\).
	Combining this with the one-dimensional estimate $\mathbb E\big[W_1(\mathsf Q_{k,n},P^k)\big]\le r_{\max}n^{-1/2}$
	yields
	$\mathbb E\big[W_2(\mathsf Q_{k,n},P^k)\big]\le \sqrt{2}\,r_{\max}\,n^{-1/4}.$ Thus the analysis remains valid with a slower empirical rate. In particular, the term in the regret bound coming from the plug-in error becomes of order
	\(
	O(T^{3/4})
	\)
	instead of \(O(\sqrt T)\). Therefore, the lower density bound is not essential for the method itself, but it yields the sharper square-root order obtained in Proposition~\ref{prop:A5_wasserstein}.
\end{remark}

\section{Experiments}\label{sec:experiments}

We evaluate the proposed method on synthetic bandit instances designed to illustrate both the behavior of our procedure and the structure of the optimal mixtures induced by the utility functional. The experiments pursue three goals: (i) to assess how well the iterates approach the simplex maximizer, (ii) to quantify the bias introduced by the plug-in influence-function estimator, and (iii) to compare the unbiased (Exact IF) and plug-in (Estimated IF) implementations of the method. Here, the Exact IF implementation denotes the analogue of our Estimated IF Algorithm~\ref{alg:if-ascent-truncated} in which the score used in the update rule is the exact influence function $\IF\left[P^{\mathsf w_t}\right]$, rather than the plug-in estimator $\widehat{\mathsf{IF}}_{t-1}=\IF\left[\widehat{\mathsf P}^{\mathsf w_t}\right]$. The Exact IF update serves as a benchmark, as it assumes knowledge of the arm distributions.

\subsection{Experimental setup}

We consider two utility functionals, namely the Variance utility and the Wasserstein utility introduced in Sections \ref{sec:ex_variance_utility} and \ref{sec:ex_wasserstein_utility} respectively.
For both utilities we study the following four scenarios:
\begin{enumerate}
	\item \textbf{Scenario 1:} $K=2$ arms with Beta reward distributions supported on $[0,1]$ and parameters $(2,2)$ and $(4,2)$.  
	\item \textbf{Scenario 2:} $K=4$ arms with Beta reward distributions supported on $[0,1]$
	and parameters equal to $(2,8)$, $(8,2)$, $(2,2)$, and $(20,20)$.
	\item \textbf{Scenario 3:} $K=8$ arms with Beta reward distributions supported on $[0,1]$
	and parameters $(a,b)=(2,1+3k)$ for $k=0,\ldots,7$. 
	\item \textbf{Scenario 4:} $K=30$ arms with Gaussian reward distributions whose means are
	chosen at random in $[-0.5,1.5]$ and whose variances are chosen at random in $[0.08,0.35]$. 
\end{enumerate}
All reported curves represent the average over 500 independent replications (or episodes), together with their standard errors. Each method is run for $T=2000$ iterations with the truncation parameter $\gamma=0.03$. For the Wasserstein utility, in all scenarios we use the uniform distribution on $[0,1]$ as a reference distribution $Q$. The initial step size was set to $\eta_0=0.5$.\footnote{The code is available here \url{https://github.com/carrascoORT/distributional-bandits}.}

For each instance we compute in advance the offline optimizer $w_\gamma^\star$ and use it as the reference solution throughout. Our main performance diagnostic is the utility gap
\(
U(w_\gamma^\star)-U(\bar{\mathsf w}_t),
\)
where $\bar{\mathsf w}_t$ denotes the mixture averaged up to time $t$.
We also report the bias diagnostic
\(
\|\boldsymbol{\mathsf B}_t\|_\infty,
\)
which measures the error induced by replacing the exact influence function with its empirical plug-in counterpart. We estimate it by means of Monte Carlo simulation with 1000 samples. More precisely, conditionally on the current pre-update state, we repeatedly sample an action-reward pair according to the current sampling distribution and the true bandit mixture, and we compare the mean plug-in stochastic gradient with its unbiased counterpart to approximate:
\(
\boldsymbol{\mathsf B}_t
=
\mathbb E\left[\widehat{\boldsymbol{\mathsf G}}_t \mid \mathcal F_{t-1}\right]
-
\mathbb E\left[\boldsymbol{\mathsf G}_t \mid \mathcal F_{t-1}\right].
\)

\begin{figure}[t]
	\centering
	\begin{subfigure}{0.49\textwidth}
		\centering
		\includegraphics[width=\linewidth]{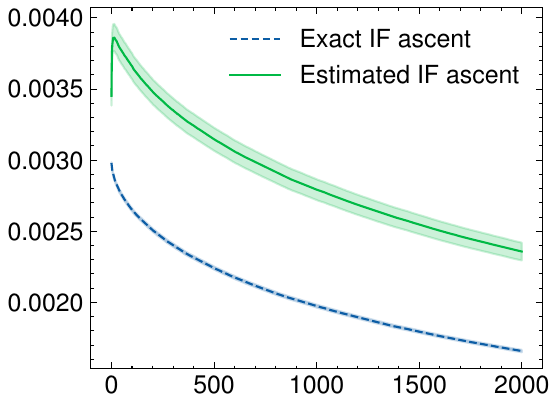}
		\caption{Scenario 1}
	\end{subfigure}
	\begin{subfigure}{0.49\textwidth}
		\centering
		\includegraphics[width=\linewidth]{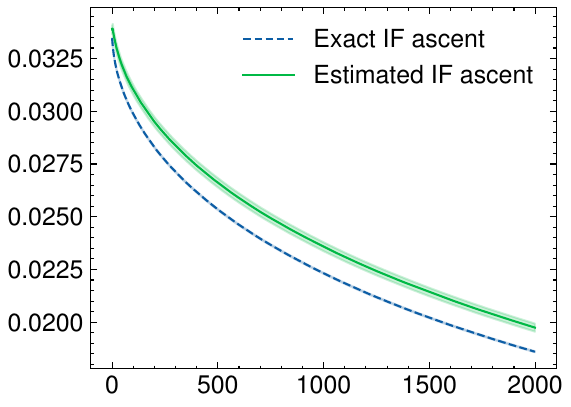}
		\caption{Scenario 2}
	\end{subfigure}
	\begin{subfigure}{0.49\textwidth}
		\centering
		\includegraphics[width=\linewidth]{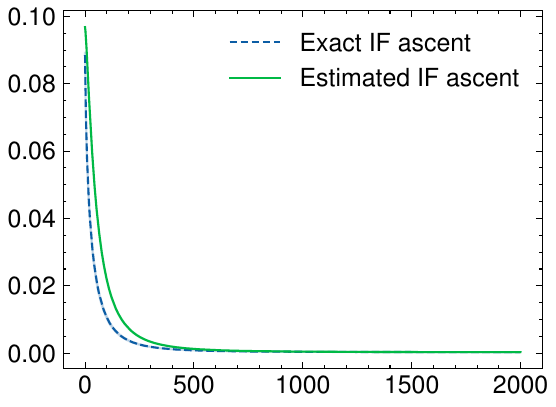}
		\caption{Scenario 3}
	\end{subfigure}
	\begin{subfigure}{0.49\textwidth}
		\centering
		\includegraphics[width=\linewidth]{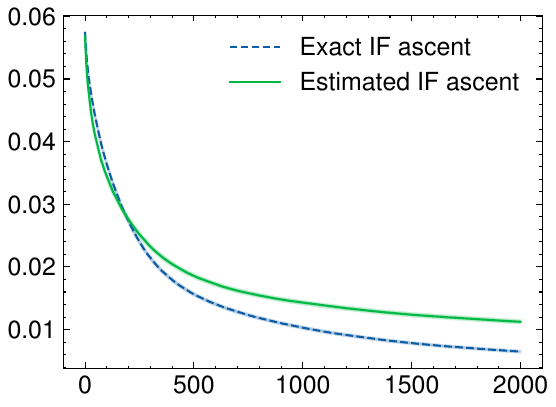}
		\caption{Scenario 4}
	\end{subfigure}
	\caption{Utility gap for the variance utility across Scenarios 1--4. Each panel reports
		the evolution of $U(w_\gamma^\star)-U(\bar w_t)$ as a function of $t$ for exact IF ascent
		and plug-in IF ascent.}
	\label{fig:var-gap}
\end{figure}

For the Estimated IF implementation, we regularize the plug-in score by means of a count-prior of size $\alpha_0>0$. In the variance case, this yields the shrunk moment estimators
\[
\widehat{\mathsf \mu}_{t,k}=\frac{\mathsf S_{t,k}+\alpha_0 m_0}{\mathsf N_{t,k}+\alpha_0},
\qquad
\widehat{\mathsf m}^{(2)}_{t,k}=\frac{\mathsf S^{(2)}_{t,k}+\alpha_0 s_0}{\mathsf N_{t,k}+\alpha_0},
\]
where $m_0$ and $s_0$ are the prior mean and prior second moment. In the Wasserstein case we use the regularized estimator
\[
\mathsf F^k_t=\frac{\mathsf N_{t,k}\mathsf P_t^k+\alpha_0 Q}{\mathsf N_{t,k}+\alpha_0}.
\]
This regularization stabilizes the plug-in score when some arms have only been sampled a few times.

\subsection{Results and discusion}

Figure~\ref{fig:var-gap} reports the evolution of the utility gap for the Variance utility in Scenarios~1--4 for both the Exact IF and the Estimated IF implementations. Figure~\ref{fig:var-bias} reports the corresponding bias diagnostic for the Estimated IF implementation. Figure~\ref{fig:wass-gap} reports the corresponding utility gap for the Wasserstein utility across Scenarios~1--4; and Figure~\ref{fig:wass-bias} presents the associated bias diagnostic. In the Wasserstein case, to complement these averaged time-evolution diagnostics, Figure~\ref{fig:wass-diagnostic} reports representative end-of-run distributional diagnostics comparing the exact optimum mixture $P_{w_\gamma^\star}$, the learned mixture $P_{\bar{\mathsf w}_T}$, and the learned mixture of arm-wise empirical distributions $\widehat{\mathsf{P}}^{\mathsf w_T}$.

\begin{figure}[t]
	\centering
	\begin{subfigure}{0.49\textwidth}
		\centering
		\includegraphics[width=\linewidth]{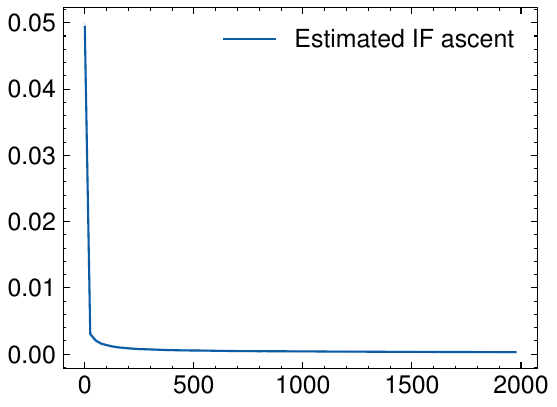}
		\caption{Scenario 1}
	\end{subfigure}
	\begin{subfigure}{0.49\textwidth}
		\centering
		\includegraphics[width=\linewidth]{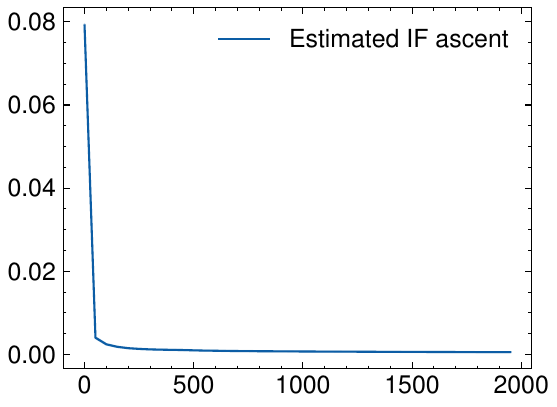}
		\caption{Scenario 2}
	\end{subfigure}
	\begin{subfigure}{0.49\textwidth}
		\centering
		\includegraphics[width=\linewidth]{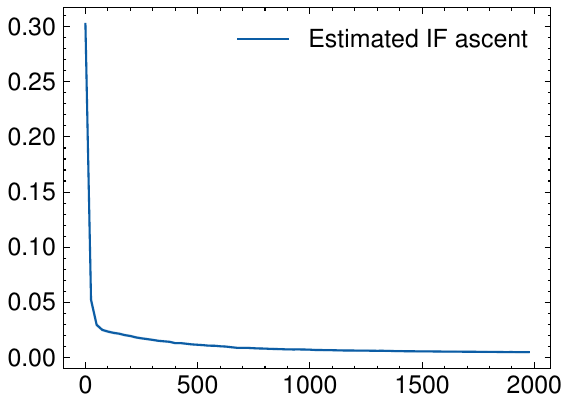}
		\caption{Scenario 3}
	\end{subfigure}
	\begin{subfigure}{0.49\textwidth}
		\centering
		\includegraphics[width=\linewidth]{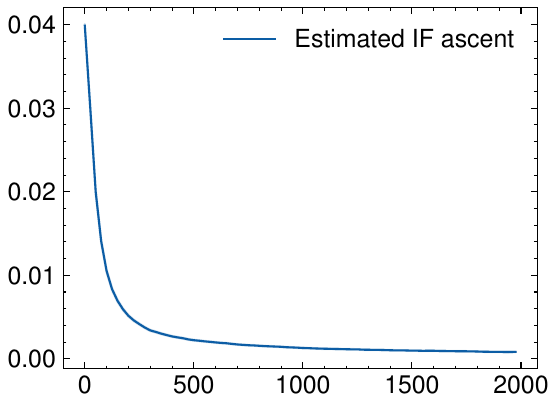}
		\caption{Scenario 4}
	\end{subfigure}
	\caption{Bias diagnostic for the variance utility across Scenarios 1--4. Each panel reports
		the evolution of $\|\boldsymbol{\mathsf B}_t\|_\infty$ for the plug-in IF estimator.}
	\label{fig:var-bias}
\end{figure}

\begin{figure}[t]
	\centering
	\begin{subfigure}{0.49\textwidth}
		\centering
		\includegraphics[width=\linewidth]{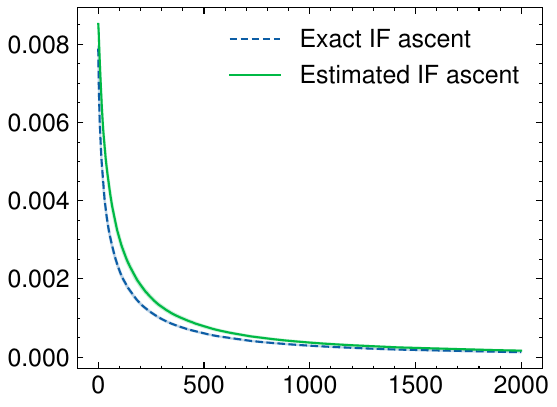}
		\caption{Scenario 1}
	\end{subfigure}
	\begin{subfigure}{0.49\textwidth}
		\centering
		\includegraphics[width=\linewidth]{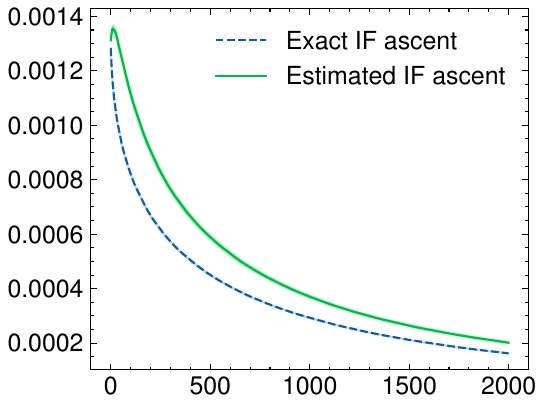}
		\caption{Scenario 2}
	\end{subfigure}
	\begin{subfigure}{0.49\textwidth}
		\centering
		\includegraphics[width=\linewidth]{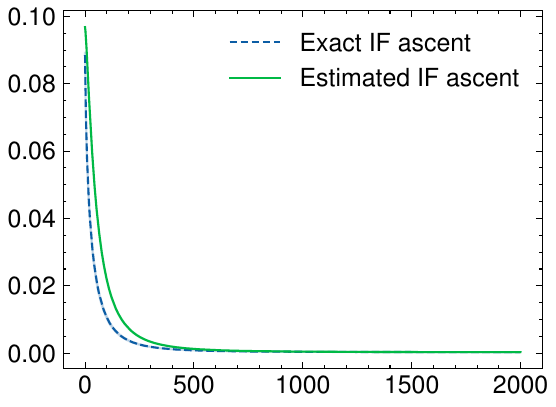}
		\caption{Scenario 3}
	\end{subfigure}
	\begin{subfigure}{0.49\textwidth}
		\centering
		\includegraphics[width=\linewidth]{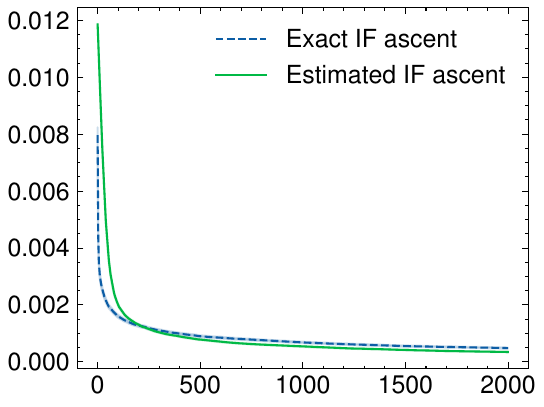}
		\caption{Scenario 4}
	\end{subfigure}
	\caption{Utility gap for the Wasserstein utility across Scenarios 1--4. Each panel reports
		the evolution of $U(w_\gamma^\star)-U(\bar w_t)$ as a function of $t$ for exact IF ascent
		and plug-in IF ascent.}
	\label{fig:wass-gap}
\end{figure}

\begin{figure}[t]
	\centering
	\begin{subfigure}{0.49\textwidth}
		\centering
		\includegraphics[width=\linewidth]{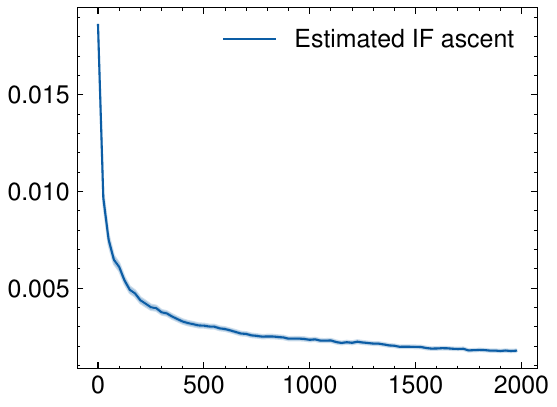}
		\caption{Scenario 1}
	\end{subfigure}
	\begin{subfigure}{0.49\textwidth}
		\centering
		\includegraphics[width=\linewidth]{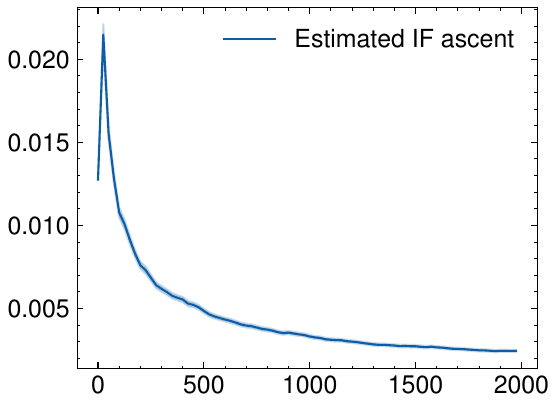}
		\caption{Scenario 2}
	\end{subfigure}
	\begin{subfigure}{0.49\textwidth}
		\centering
		\includegraphics[width=\linewidth]{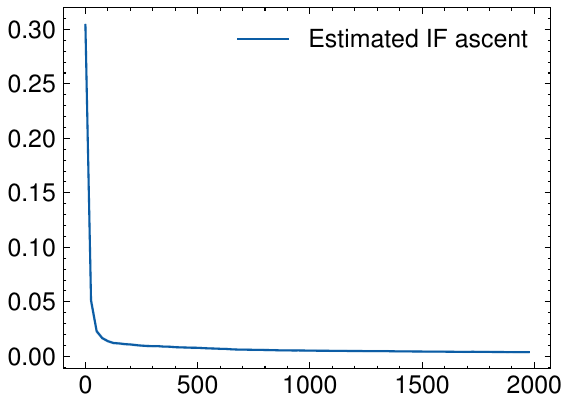}
		\caption{Scenario 3}
	\end{subfigure}
	\begin{subfigure}{0.49\textwidth}
		\centering
		\includegraphics[width=\linewidth]{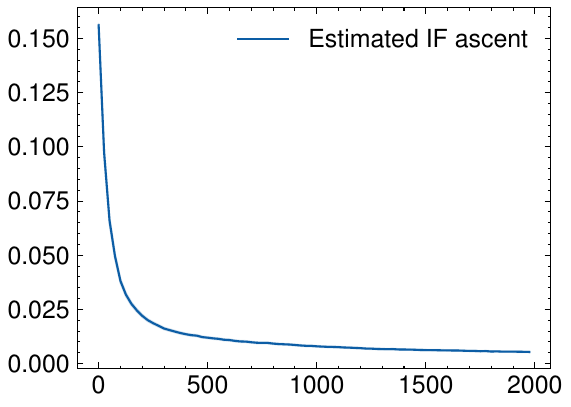}
		\caption{Scenario 4}
	\end{subfigure}
	\caption{Bias diagnostic for the Wasserstein utility across Scenarios 1--4. Each panel
		reports the evolution of $\|\boldsymbol{\mathsf B}_t\|_\infty$ for the plug-in IF estimator.}
	\label{fig:wass-bias}
\end{figure}

\begin{figure}[t]
	\centering
	\begin{subfigure}{0.49\textwidth}
		\centering
		\includegraphics[width=\linewidth]{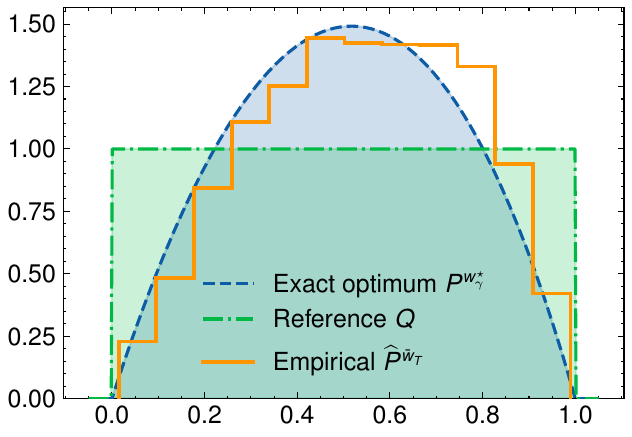}
		\caption{Scenario 1}
	\end{subfigure}
	\centering
	\begin{subfigure}{0.49\textwidth}
		\centering
		\includegraphics[width=\linewidth]{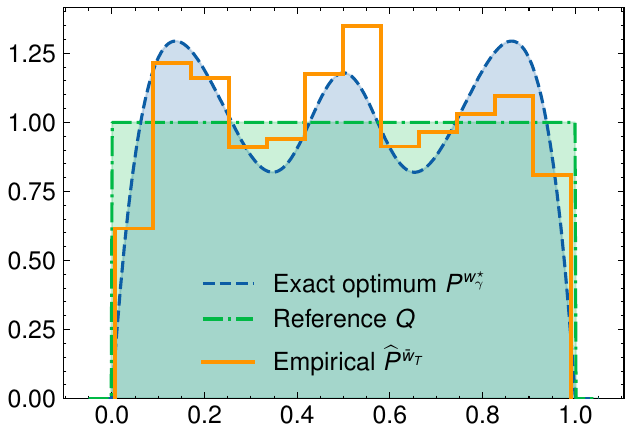}
		\caption{Scenario 2}
	\end{subfigure}
	\begin{subfigure}{0.49\textwidth}
		\centering
		\includegraphics[width=\linewidth]{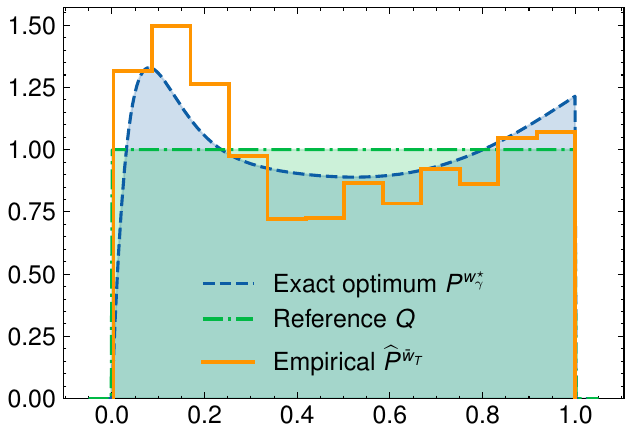}
		\caption{Scenario 3}
	\end{subfigure}
	\begin{subfigure}{0.49\textwidth}
		\centering
		\includegraphics[width=\linewidth]{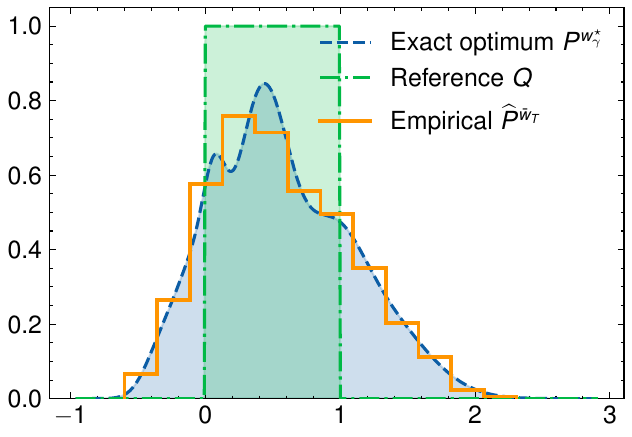}
		\caption{Scenario 4}
	\end{subfigure}
	\caption{Distributional diagnostic for the Wasserstein utility. Each panel compares the
		oracle mixture distribution $P_{w_\gamma^\star}$, the learned mixture $P_{\bar w_T}$, and
		the empirical distribution $\widehat P_T$ obtained from the observed rewards. The oracle
		and learned mixtures are shown through their densities, while $\widehat P_T$ is represented
		by a kernel density estimate.}
	\label{fig:wass-diagnostic}
\end{figure}

The experiments support three main conclusions. First, for both utilities and in all four scenarios, the utility gap decreases over time for both implementations, showing that the proposed update rule consistently drives the iterates toward the truncated optimum $U_\gamma^\star$. Second, the plug-in Estimated IF implementation based on the empirical influence function is generally competitive with the Exact IF benchmark, although it typically incurs some loss. Third, the bias diagnostic $\|\boldsymbol{\mathsf B}_t\|_\infty$ explains part of the discrepancy between the two methods, but not all of it: the practical performance of the plug-in method depends not only on the magnitude of the score-estimation error, but also on the geometry of the utility and on the structure of the underlying bandit instance.

For the variance utility, Figure~\ref{fig:var-gap} shows a stable decay of the utility gap in all four scenarios. The Exact IF benchmark performs best throughout, although the separation from the Estimated IF implementation is small in some cases. This is particularly clear in Scenario~3, where the two curves become very close after the initial transient. By contrast, Scenarios~1 and~4 exhibit a more persistent gap, with Scenario~4 providing the clearest example of a sustained performance loss for the plug-in method. Scenario~2 lies in between: both methods improve steadily, but the Exact IF implementation retains a visible advantage over the full horizon.

In all four variance scenarios, the bias decreases rapidly after a short transient and eventually becomes small. However, the utility-gap difference between Exact IF and Estimated IF does not disappear uniformly once the bias is small. This is most visible in Scenarios~1 and~4, where a non-negligible performance gap remains despite the decay of $\|\boldsymbol{\mathsf B}_t\|_\infty$. This suggests that the error induced by plug-in estimation is not captured solely by the instantaneous magnitude of the bias diagnostic, and that noise accumulation and the underlying geometry also play a significant role.

For the Wasserstein utility, Figure~\ref{fig:wass-gap} reveals the same overall convergence pattern, but with a stronger sensitivity to score-estimation error. In Scenarios~1 and~2, the Estimated IF implementation remains close to the Exact IF benchmark, but is consistently worse. Scenario~3 displays the clearest separation at early times, although the gap narrows substantially later on. Scenario~4 is more nuanced: after a visible initial discrepancy, the plug-in method catches up and eventually becomes slightly better than the Exact IF benchmark. Taken together, these experiments show that the implementable Estimated IF method remains effective for the Wasserstein utility, although its performance is more scenario-dependent than in the variance case.

Figure~\ref{fig:wass-bias} supports this reading. For the Wasserstein utility, the bias is typically decays more slowly than for the variance utility, especially in Scenarios~2 and~4. This is consistent with the greater difficulty of estimating Wasserstein-based scores from finite samples. At the same time, the relation between bias magnitude and utility performance is not monotone. In particular, a comparatively large bias does not necessarily prevent good optimization performance, as illustrated by Scenario~4. Therefore, while the bias diagnostic is clearly informative, it should be interpreted as only one component in the explanation of the observed behavior.

Finally, the distributional diagnostics in Figure~\ref{fig:wass-diagnostic} show that the learned mixtures often recover the main qualitative features and tracks well the coarse shape of the true optimal mixture. These plots should therefore be interpreted qualitatively, since they illustrate representative end-of-run distributions rather than averaged trajectories. Their main role is to show that the method does not only approach the optimal utility value, but also captures approximately the distributional shape selected by the Wasserstein criterion.

Overall, the experiments show that the proposed influence-function mirror-ascent procedure is effective across a range of synthetic bandit scenarios and for two qualitatively different utility functionals. 

\section{Concluding remarks}\label{sec:conclusion}

We have introduced a general framework for stochastic multi-armed bandits whose objective is a concave statistical utility of the long-run reward distribution. The key structural observation is that, under mild continuity assumptions, the infinite-horizon problem can be reduced to the optimization of stationary mixed policies. This transforms the original sequential problem into a concave optimization problem over the probability simplex, where each sampling distribution induces a mixture of the arm reward laws.

To optimize this objective from bandit feedback, we developed an influence-function approach to gradient estimation. The resulting estimators provide first-order information for nonlinear distributional utilities and lead naturally to an entropic mirror-ascent algorithm on a truncated simplex. Our regret analysis separates the optimization error of the mirror-ascent procedure from the statistical bias introduced by plug-in estimation of the influence function. This decomposition makes explicit how the quality of the influence-function estimator affects the overall performance of the algorithm.

The variance and Wasserstein examples illustrate how the framework applies to concrete distributional objectives beyond the classical mean-reward criterion. Several extensions remain open. One natural direction is to analyze further classes of statistical utilities, including tail-sensitive, risk-sensitive, or information-theoretic criteria. Another important direction is to weaken the compact-support assumptions used in parts of the analysis. We expect that for many utilities of interest, compact support can be replaced by suitable moment, integrability, or tail conditions, possibly combined with truncation or robustification arguments. Other future directions include sharper rates for specific utilities, adaptive choices of the truncation parameter, high-probability regret bounds, and extensions to contextual or nonstationary settings. More broadly, the influence-function perspective suggests a systematic way to design bandit algorithms for statistical objectives that depend on the full reward distribution rather than on its expectation alone.

\bibliographystyle{plainnat}

\bibliography{biblio}

\end{document}